%% file: main.tex
\def\comments{} 

\def\year{2022}\relax
\documentclass[letterpaper]{article} 
\usepackage{aaai22}  
\usepackage{times}  
\usepackage{helvet}  
\usepackage{courier}  
\usepackage[hyphens]{url}  
\usepackage{graphicx} 
\usepackage{natbib}  
\usepackage{caption} 
\DeclareCaptionStyle{ruled}{labelfont=normalfont,labelsep=colon,strut=off} 
\usepackage{algorithm}
\usepackage{algorithmic}

\usepackage{newfloat}
\usepackage{listings}
\floatstyle{ruled}
\newfloat{listing}{tb}{lst}{}
\floatname{listing}{Listing}

\usepackage[utf8]{inputenc} 
\usepackage[T1]{fontenc}    
\usepackage{hyperref}       
\usepackage{url}            
\usepackage{booktabs}       
\usepackage{amsfonts}       
\usepackage{nicefrac}       
\usepackage{microtype}      
\usepackage{xcolor}         

\usepackage{subcaption}

\usepackage{colortbl}
\usepackage{import}

\usepackage{makecell, cellspace}
\usepackage{wrapfig}

\usepackage[page]{appendix}

\usepackage{amsthm}
\usepackage{amsmath,amssymb}
\usepackage{amsfonts}
\usepackage{bbding}

\usepackage{mathtools}
\usepackage{algorithm}
\usepackage{algorithmic}

\usepackage{tikz}
\usepackage{multirow}
\usetikzlibrary{shapes.geometric, arrows, automata, positioning}
\usepackage{makecell}

\usepackage{dsfont}
\usepackage{enumerate}   
\usepackage{natbib}

\DeclarePairedDelimiter\br{(}{)}
\DeclarePairedDelimiter\brs{[}{]}
\DeclarePairedDelimiter\brc{\{}{\}}
\DeclarePairedDelimiter\inner{\langle}{\rangle}
\DeclarePairedDelimiter\abs{\lvert}{\rvert}

\usepackage{verbatim}

\usepackage{bbm}

\DeclareMathOperator*{\argmax}{arg\,max}

\DeclareMathOperator*{\argmin}{arg\,min}

\DeclareMathOperator{\reals}{\mathbb{R}}

\makeatletter
\newcommand*{\rom}[1]{\expandafter\@slowromancap\romannumeral #1@}
\makeatother

\def\M{{\mathcal M}}

\newcommand{\sset}{{\mathcal S}}
\newcommand{\aset}{{\mathcal A}}
\newcommand{\F}{{\mathcal F}}

\newcommand*\dkl[2]{d_{KL}(#1||#2)}
\newcommand*\bregman[2]{B_\omega\left(#1,#2\right)}

\newcommand{\A}{{\bf A}}

\newcommand{\C}{{\mathcal C}}

\newcommand{\E}{{\mathbb E}}

\newcommand{\regret}{{\text{Reg}}}

\def\triangleq{:=}

\usepackage{dsfont}

\DeclareMathOperator{\ssimplex}{\Delta_\A^S}

\newcommand*\norm[1]{\left\|#1\right\|}

\usepackage{array}
\newcolumntype{P}[1]{>{\centering\arraybackslash}p{#1}}
\newcolumntype{L}[1]{>{\raggedright\arraybackslash}p{#1}}

\usepackage[capitalize]{cleveref}
\crefname{equation}{eq.}{eqs.}
\Crefname{equation}{Eq.}{Eqs.}

\usepackage{thmtools}
\usepackage{thm-restate}

\newtheorem{lemma}{Lemma}
\newtheorem{corollary}{Corollary}

\newtheorem{definition}{Definition}

\numberwithin{equation}{section}
\numberwithin{remark}{section}

\ifx\comments\undefined
    \newcommand{\lior}[1]
        {}
        \newcommand{\tom}[1]
        {}
        \newcommand{\shie}[1]
        {}
        \newcommand{\todo}[1]
        {}
        \newcommand{\yon}[1]
        {}
\else
    \newcommand{\lior}[1]
    {{\color{blue} Lior: #1}}
    \newcommand{\tom}[1]
    {{\color{purple} Tom: #1}}
    \newcommand{\shie}[1]
    {{\color{green} Shie: #1}}
    \newcommand{\todo}[1]
    {{\color{red} ToDO: #1}}
    \newcommand{\yon}[1]
    {{\color{gray} yon: #1}}
\fi

\newenvironment{proofsketch}{%
  \proof}{\endproof}

\definecolor{darkgray}{rgb}{0.66, 0.66, 0.66}

\renewcommand{\appendixtocname}{List of Appendices}

\makeatletter
\let\oldappendix\appendices

\renewcommand{\appendices}{%
  \clearpage
  \renewcommand{\thesection}{\Alph{section}}
  \let\tf@toc\tf@app
  \addtocontents{app}{\protect\setcounter{tocdepth}{2}}
  \immediate\write\@auxout{%
    \string\let\string\tf@toc\string\tf@app^^J
  }
  \oldappendix
}%

\newcommand{\listofappendices}{%
  \begingroup
  \renewcommand{\contentsname}{\appendixtocname}
  \let\@oldstarttoc\@starttoc
  \def\@starttoc##1{\@oldstarttoc{app}}
  \tableofcontents
  \endgroup
}

\makeatother

\usepackage[symbol]{footmisc}

\renewcommand{\algorithmiccomment}[1]{\textit{#1}}

\title{Online Apprenticeship Learning}
\author {
    Lior Shani\textsuperscript{\rm 1},
    Tom Zahavy\textsuperscript{\rm 2},
    Shie Mannor\textsuperscript{\rm 1,3}
}
\affiliations {
    \textsuperscript{\rm 1} Technion – Israel Institute of Technology, Israel\\
    \textsuperscript{\rm 2} Deepmind, UK\\
    \textsuperscript{\rm 3} Nvidia Research, Israel\\
    shanlior@gmail.com, tomzahavy@gmail.com, shie@ee.technion.ac.il
}

\usepackage{bibentry}

\begin{document}

\maketitle

\begin{abstract}
In Apprenticeship Learning (AL), we are given a Markov Decision Process (MDP) without access to the cost function. Instead, we observe trajectories sampled by an expert that acts according to some policy. The goal is to find a policy that matches the expert's performance on some predefined set of cost functions. 
We introduce an online variant of AL (Online Apprenticeship Learning; OAL), where the agent is expected to perform comparably to the expert while interacting with the environment. We show that the OAL problem can be effectively solved by combining two mirror descent based no-regret algorithms: one for policy optimization and another for learning the worst case cost. By employing optimistic exploration, we derive a convergent algorithm with $O(\sqrt{K})$ regret, where $K$ is the number of interactions with the MDP, and an additional linear error term that depends on the amount of expert trajectories available. Importantly, our algorithm avoids the need to solve an MDP at each iteration, making it more practical compared to prior AL methods. Finally, we implement a deep variant of our algorithm which shares some similarities to GAIL \cite{ho2016generative}, but where the discriminator is replaced with the costs learned by the OAL problem. Our simulations suggest that OAL performs well in high dimensional control problems.


\end{abstract}

\section{Introduction}

In Reinforcement Learning \citep[RL]{sutton2018reinforcement} an agent interacts with an environment by following a policy. The environment is modeled as a Markov Decision Process \citep[MDP]{puterman1994markov}, where in each state, the agent takes an action based on the policy, and as a result, pays a cost and transitions to a new state. The goal of RL is to learn an optimal policy that minimizes the long term cumulative cost. This makes RL useful when we can specify the MDP model appropriately.
However, in many real-world problems, it is often hard to define a cost which induces the desired behaviour. E.g., an autonomous driver might suffer costs when driving slowly or in a hazardous way. Yet, prescribing these costs can be eluding.


A feasible solution to this problem is \textit{Imitation Learning} (IL). This setup introduces the notion of an expert, typically a human, that provides us with a set of demonstrations. The agent's goal is to learn the optimal policy by imitating the expert’s decisions. Methods such as \textit{Behavioural Cloning} (BC) try to directly mimic the demonstrator by applying a supervised learning (SL) algorithm to learn a mapping from the states to actions. This literature is too vast to cover here and we refer the reader to \cite{schaal1997learning, argall2009survey}. 


\textit{Apprenticeship Learning} \citep[AL]{abbeel2004apprenticeship} aims to address the same motivation using a different goal. Rather than learning a cost, its goal is to find a policy whose performance is close to that of the expert for \emph{any possible cost} in a known set. This keeps the state-action occupancy of the agent and expert in proximity, requiring the AL agent to find a path back to the expert trajectories in states that are unobserved by the expert. This differs from BC, in  which the agent's policy is undetermined in these unobserved states.
Prior works on AL \cite{abbeel2004apprenticeship, syed2008game, zahavy2019apprenticeship} mostly considered a batch RL setting with the purpose of finding an $\epsilon$-optimal solution, where the transition model is typically known or can be extracted from the data. However, in many real world applications, the model is unknown, and the learner is inflicted costs when performing poorly on the task, even if these costs are not properly specified to serve as an objective.
This leads us to consider an online version of AL in which an agent should perform as close as possible to the expert on any possible cost, \emph{while it is learning}. As a result, an \emph{online} autonomous driver would try to imitate the expert \emph{when learning in the real-world}, avoiding unnecessary costs.

AL is typically formulated as a min-max game between a policy and a cost ``players''. This problem was shown to be convex in the cost and in the feature expectations of the policy, but not in the policy itself. 
Existing AL algorithms proposed to bypass this issue via the concept of best response. \citet{syed2008apprenticeship} proposed the projection algorithm, in which the policy player applies the best response and the cost player uses Mirror Descent. Alternatively, in MWAL \citep{abbeel2004apprenticeship}, the cost player plays the best response and the policy player plays a Frank-Wolfe step \citep{zahavy2019apprenticeship} by utilizing the convexity in the feature expectations. Unfortunately, this requires both algorithms to solve an MDP in each iteration (see \cref{sec:al}). 


Instead, in the convex games setting, 
 min-max games can be approximately solved by simultaneously running two competing no-regret algorithms, preventing the inefficiency of finding the best response in each iteration \cite{abernethy2017frank}. This result builds on the notion of stability found in online convex optimization algorithms such as Mirror Descent \citep[MD]{beck2003mirror}. Interestingly, there has been a recent body of papers connecting policy optimization techniques and online convex optimization. Specifically, in \cite{geist2019theory, shani2020adaptive}, the authors prove global convergence for an MD-based policy optimization algorithm.
 Moreover, in \cite{cai2019provably,efroni2020optimistic}, the authors show that using Mirror Descent policy optimization together with optimistic exploration leads to no-regret policy optimization algorithms.

In this work, we take a similar approach and propose an online AL algorithm (OAL, pronounced Owl) that minimizes the AL regret: the difference between the agent's cumulative performance and that of the expert, for any possible cost. Our algorithm performs a dual MD step in which, (1) the policy is improved using one step of \emph{optimistic} policy optimization MD-based update, and (2) the cost is updated by a single MD iteration. 
We show this leads to a \emph{sample efficient} algorithm, even when the model is unknown. Importantly, our algorithm avoids solving an MDP in each iteration, making it \emph{more practical} than previous approaches. Finally, we conduct an empirical study to verify the need for exploration in OAL.

To illustrate the benefits and practicality of our approach, we implement a deep RL variant of OAL, based on the Mirror Descent Policy Optimization \citep[MDPO]{tomar2020mirror} algorithm. Our deep OAL variant holds connection to the Generative Adversarial IL algorithm \citep[GAIL]{ho2016generative}: both OAL and GAIL use a generative cost function and take a single policy improvement step in each iteration. Differently from GAIL which learns a probabilistic discrimination between the policy and expert, OAL aims to optimize the value difference between the policy and expert, based on the min-max formulation. This is closely related to GAIL variants based on the Wasserstein distance \cite{xiao2019wasserstein, chen2020computation}. Our experiments on continuous control tasks suggest that OAL is comparable to GAIL. 

\section{Preliminaries}\label{sec: preliminaries}
In this work, we will deal with finite-horizon MDPs, defined by a tuple $ \M \triangleq \br*{\sset,\aset,p,c,H}$, where $\sset,\aset$ are the state and action spaces, respectively, and $H$ is the length of an episode. $p_h(s' \mid s , a)$ is transition kernel describing the probability of transitioning to any state $s'$, given the current state $s$ and action $a$, for any $h\in [H]$. Similarly, $c_h(s,a)$ is the cost of applying action $a$ at state $s$, during the $h$-th time-step. In adversarial MDPs, we allow the costs to change arbitrarily between episodes. A policy  $\pi_h:\sset\rightarrow\aset$ is a mapping from state to action. The value function $V^{\pi,p,c}_h(s)=\E \brs*{ \sum_{t=h}^H c_h(s,a) \mid  s_h=s,\pi}$ is the cumulative expected costs of the agent, following $\pi$ from state $s$ at time-step $h$, over the MDP defined by the transition kernel $p$ and costs $c$. Similarly, we define the $Q$-function, $Q^{\pi,p,c}_h(s,a)=\E \brs*{ \sum_{t=h}^H c_h(s,a) \mid  s_h=s, a_h=a,\pi}$. 
The occupancy measure
   $ d_h^{\mu,\pi,p}(s,a)=\Pr\br*{s_h=s,a_h=a \mid \mu, \pi,p}$ is the probability to reach state $s$ and action $a$, at the $h$-th timestep, following $\pi$ and starting from the initial distribution $\mu$. Here throughout, we omit $\mu$ and assume without loss of generality that there exists a single starting state. Also, we omit $p$ when clear from context. Notably, it holds for any $\pi$, that $
    \E^{s\sim \mu} \brs*{V_1^{\pi}(s)}=\inner*{ c,d^{\pi}}$,
where $\inner*{ c,d^{\pi}}\!\triangleq \!\sum_{h=1}^H \inner*{c_h, d_h^\pi} \!= \! \sum_{h, s,a} c_h(s,a) d_h^\pi(s,a)$.
A mixed policy $\psi$ over the set of all deterministic policies $\Pi^{det}$ is executed by randomly selecting the policy $\pi_i \in \Pi^{det}$ at the beginning of an episode with probability $\psi(i)$, and exclusively following $\pi_i$ thereafter.
Finally, the filtration $\mathcal{F}_k$ includes all events 
in the $k$-th episode.
We omit logarithmic factors when using the $O(\cdot)$ notation.

\subsection{Mirror Descent in RL}\label{sec: MD in RL}
The role of conservative updates in the  convergence of policy optimization algorithms has been extensively studied in RL, going back to the analysis of the Conservative Policy Iteration (CPI) algorithm \cite{kakade2002approximately}. Though sometimes motivated differently, the notion of conservative or stable updates is deeply related to ideas and analyses found in the convex optimization literature. Specifically, CPI can be considered an RL variant of the Frank-Wolfe (FW) algorithm \cite{scherrer2014local}. Alternatively, the MD algorithm was also studied and applied to MDPs, allowing to provide theoretical guarantees for RL algorithms \cite{geist2019theory, shani2020adaptive}. 

MD \cite{beck2003mirror} is a framework for solving convex optimization problems. At each iteration, the MD procedure minimizes the sum of a linear approximation of the current objective and a Bregman divergence term, aimed to keep consecutive iterates in proximity. For a set ${f_k}$ of convex losses, and a constraint set $\C$, the $k$-th MD iterate is $x_{k+1} \in \argmin_{x \in \C} \inner*{\nabla f_k(x) \rvert_{x=x_k} , x- x_k} + t_k \bregman{x}{x_k} $,
where $B_\omega$ is a Bregman divergence and $t_k$ is a step size. Finally, MD is known to be a no-regret online optimization algorithm \cite{hazan2019introduction}. More formally, $\regret(K) \triangleq  \max_x \sum_{k=1}^{K} f_k(x_k) - f_k(x) \leq O(\sqrt{K}).$

The stability of the MD updates is crucial to obtain $O(\sqrt{K})$ regret in online optimization, where at each iteration the learner encounters an arbitrary loss function \cite{hazan2019introduction}. This property was also exploited in RL to prove convergence in Adversarial MDPs, where the costs can change arbitrarily between episodes. Indeed, \citet{neu2010online} provided such guarantees when the transition model is known. Recently, in \cite{cai2019provably, efroni2020optimistic}, the authors provided convergent MD policy optimization algorithms for adversarial MDPs when the model is unknown. These algorithms perform an \emph{optimistic policy evaluation step to induce exploration}, followed by an MD policy update.

The benefits of MD in RL go beyond establishing convergence guarantees. \citet{shani2020adaptive} shows that TRPO \cite{schulman2015trust}, a widely used practical deep RL algorithm is actually an adaptation of the MD algorithm to MDPs. As a result, \citet{tomar2020mirror} derived Mirror Descent Policy Optimization (MDPO), a closer-to-theory deep RL algorithm based on the re-interpretation of TRPO, with on-policy and off-policy variants.



\subsection{Apprenticeship Learning}
\label{sec:al}
In AL, we assume the existence of an \emph{expert policy}, denoted by $\pi^E$. We assume access to $N$ experts' trajectories sampled from $\pi^E$ over the MDP, from which we construct an estimate of the occupancy measure $d^E$, denoted by $\hat d^E$. While the cost is unknown in AL, we assume it belongs to some set of costs $\C$. 
In the theoretical analysis, we  focus on the following set:

$\C_b$ -- \textbf{Bounded costs:} In this tabular case (\Cref{eq: AL}), the costs are of the form $c_h(s,a)\in[0,1]$, $\forall h,s,a$. 

In our experiments, we will refer to the following sets:

$\C_l$ -- \textbf{Linear costs:} The states are assumed to be associated with features $\phi(s) \in [-1,1]^d$, and $\C_l$ is the costs that are linear in the features: i.e., $c(s)= w \cdot \phi(s)$. For any $w \in \mathcal{W}$, $w$ is usually assumed to be the $\ell_2$ unit ball \citep{abbeel2004apprenticeship} or the simplex \citep{syed2008game}. The feature expectations of a policy $\pi$ are $\Phi^\pi \triangleq \mathbb{E}_{d^\pi}\phi(s).$

$\C_n$ -- \textbf{Non-linear costs:} In this case the costs are some general non-linear function (typically a DNN) of the state features: $c(s)=f(\phi(s))$, where $f$ is bounded. We will also consider the case that $f$ is lipschitz continuous. In this case, AL is related to minimizing the Wasserstein distance between the agent and the expert \cite{zhang2020wasserstein,zhang2020generative}. 

The goal of AL is to find a policy $\pi$ with good performance, relative to the expert, for any possible cost within a set $\C$, 
\begin{equation}
    \label{eq: AL}
    \text{AL:} \enspace \argmin_{\pi} \max _{c\in\C}  \enspace \inner*{c, d^\pi} - \inner{ c, d^E  }.
\end{equation}
Previous works mostly focus on the space of mixed policies, and linear costs (see \cref{supp: AL discussion} for a discussion on differences between the tabular and linear setting). In this case, \cref{eq: AL} is equivalent to 
$ \argmin_{\psi\in\Psi} \max _{w\in \mathcal{W}}  \enspace  \inner*{w, \Phi^\psi} - \inner*{w,\Phi^E} $.
\citet{abbeel2004apprenticeship} analyzed this objective when $\mathcal{W}$ is the euclidean unit ball. In this setup, it is possible to compute the \textbf{best response} for the cost (the maximizer over $\mathcal{W}$), exactly, for any $\psi$, and get that $w = \frac{\Phi^\psi - \Phi^E}{\|\Phi^\psi - \Phi^E\| }$. Plugging this back in the objective, we get that solving \cref{eq: AL} is equivalent to Feature Expectation Matching (FEM), i.e., minimizing
$ \| \Phi^\psi - \Phi^E\|^2 $. To solve the FEM objective, the authors propose the projection algorithm. This algorithm starts with an arbitrary policy $\pi_0$ and computes its feature expectations $\Phi^{\pi_0}$. At step $t$ they fix a cost $w_t = \bar \Phi_{t-1} - \Phi^E$ and find the policy $\pi_t$ that minimizes it, where $\Bar \Phi_t$ is a convex combination of the feature expectations of previous (deterministic) policies $\bar{\Phi}_t = \sum^t _{j=1}\alpha_j \Phi^{\pi_j}.$ They show that in order to get that $\norm{\bar{\Phi}_T-\Phi_E}\le \epsilon$, it suffices to run the algorithm for $O(\frac{d}{\epsilon^2}\log(\frac{d}{\epsilon}))$ iterations (where $d$ is features dimension).


Another type of algorithms, based on online convex optimization, was proposed by \citet{syed2008game}. Here, the cost player plays a no-regret algorithm and the policy player plays the best response, i.e., it plays the policy $\pi_t$ that minimizes the cost at time $t$. The algorithm runs for $T$ steps and returns a mixed policy $\psi$ that assigns probability $1/T$ to each policy $\pi_t$. In \cite{syed2008game}, the authors prove their scheme is faster than the projection algorithm \cite{abbeel2004apprenticeship}, requiring only $O(\log(d)/\epsilon^2)$ iterations. This improvement follows from the analysis of MD and specifically the Multiplicative Weights algorithm \citep{freund1997decision,littlestone1994weighted}, giving the algorithm its name, Multiplicative Weights AL (MWAL). 


Both types of AL algorithms we have described are based on the concept of solving the min-max game when one of the players plays the best response: the policy player in MWAL, and the cost player in the projection algorithm. The main limitation in implementing these algorithms in practice is that they both require to solve an MDP in each iteration.

\section{Online Apprenticeship Learning}\label{sec: online apprenticeship learning}

In this work, we study an online version of AL where an agent interacts with an environment with the goal of imitating an expert. Our focus is on algorithms that are sample efficient in the number of interactions with the environment. This is different from prior batch RL work \cite{abbeel2004apprenticeship, syed2008game, zahavy2019apprenticeship} which mostly focused on PAC bounds on the amount of optimization iterations needed to find an $\epsilon$-optimal solution and typically assumed that the environment is known (or that the expert data is sufficient to approximate the model). Our formulation, on the other hand, puts emphasis on the performance of the agent \emph{while it is learning}, which we believe is important in many real world applications.

Formally, we measure the performance of an online AL algorithm via the \emph{regret} of the learning algorithm w.r.t the expert. In standard RL, when the costs are known, the regret of a learner is defined as the difference between the expected accumulated values of the learned policies and the value of the optimal policy \cite{jaksch2010near}. However, in the absence of costs, the optimal policy is not defined. Therefore, it is most natural to compare the performance of the learner to the expert. With \Cref{eq: AL} in mind, this leads us to introduce the regret in \Cref{def: regret}, which measures the \emph{worst-case} difference between the accumulated values of the learner and the expert, \emph{over all possible costs} in $\C$:
\begin{definition}[Apprenticeship Learning Regret]\label{def: regret}
The regret of an AL algorithm is:
\begin{align}\label{eq: regret of AL}
\regret_{AL}(K) \triangleq  \max_{c\in \C} \sum_{k=1}^K \brs*{ V_1^{\pi_k, c} - V_1^{ \pi^E, c}},
\end{align}
\end{definition}
\Cref{def: regret} suggests a notion of regret from the perspective of comparison to the expert as a reference policy. Instead, as an optimization problem, the regret of~\eqref{eq: AL} is measured w.r.t. to its optimal solution,
$ \regret(K) \triangleq  \max_{c\in\C} \sum_{k=1}^K \langle c, d^{\pi^k}  - d^E \rangle - \min_{\pi}  \max_{c\in\C}  \sum_{k=1}^K \inner*{c, d^{\pi} - d^E}$.
Importantly, in the following lemma, we show the two regret definitions coincide:

\begin{restatable}{lemma}{RegretEquivalency}\label{lemma: regrets equivalence}
The online regret of the AL optimization problem~\eqref{eq: AL} and the AL regret are equivalent.
\end{restatable}




\subsection{Online Apprenticeship Learning Scheme}

\begin{algorithm}
\caption{OAL Scheme}\label{alg: OAL scheme}
\begin{algorithmic}[1]
\FOR{$k=1,...,K$}
    \STATE Rollout a trajectory by acting $\pi_k$
    \STATE {\color{gray} \# Evaluation Step}
    \STATE Evaluate $ Q^{\pi_k}$ using the current cost $c^k$
    \STATE Evaluate $  \nabla_{c} L(\pi^k, c; \pi^E) \rvert_{c=c^k}$ 
    \STATE {\color{gray} \# Policy Update}
    \STATE Update $\pi^{k+1}$ by an MD policy update with $ Q^{\pi^k}$
    \STATE {\color{gray} \# Costs Update}
    \STATE Update $c^{k+1}$ by an MD step on $ \nabla_{c} L(\pi^k, c)\rvert_{c=c^k}$
\ENDFOR
\end{algorithmic}
\end{algorithm}

In \Cref{alg: OAL scheme}, we present a scheme for solving the AL problem using online optimization tools. Specifically, we introduce a min-player to solve the minimization problem in \cref{eq: AL}. This min-player is an RL agent that aims to find the optimal policy in an adversarial MDP in which a max-player chooses the cost in each round. In \Cref{sec: convergence of oal}, we show that simultaneously optimizing both the policy and cost using no-regret algorithms leads to a \emph{sample efficient} no-regret AL algorithm (see \Cref{def: regret}). Our approach \emph{averts the need to solve an MDP in each iteration}, as was typically done in previous work (see the discussion in \Cref{sec:al}), and therefore vastly reduces the computational complexity of the algorithm and makes it more practical. 

This is attained in the following manner. Each OAL iteration consists of two phases:  \textbf{(1) evaluation phase}, in which the gradients of the objective w.r.t. the policy and costs are estimated, and \textbf{(2) optimization phase}, where both the policy and cost are updated by two separated MD iterates (see \Cref{sec: MD in RL}).
To specify the updates, we need calculate the gradient of the AL objective w.r.t. to the policy or cost, and choose an appropriate Bregman divergence. 

\textbf{Policy update.}
Because $V^{\pi^E,c^k}$ does not depend on the current policy, the optimization objective is just $V^{\pi,c^k}$, which is the exactly the RL objective w.r.t. to the current costs. Thus, the gradient of the AL objective w.r.t. policy is the $Q$-function of the current policy and costs. The KL-divergence is a natural choice for the Bregman term, when optimizing over the set of stochastic policies \cite{shani2020adaptive}. Using the stepsize $t^\pi_k$, the OAL policy update is
\begin{align}
    &\pi_h^{k+1} \in \argmin_\pi \inner*{Q_h^{\pi_k,c_k}, \pi_h} + t^\pi_k \dkl{\pi_h}{\pi_h^k}. \label{eq: OAL MD policy update}
\end{align}
Notably, this update only requires to evaluate the current $Q$-function, and does not to solve an MDP.

\textbf{Cost update.} Denoting the cost AL objective $L(\pi, c) \triangleq - (V_1^{\pi,c}(s_1) -  V_1^{\pi^E,c}(s_1)) $, the gradient w.r.t. the cost is $\nabla_c L(\pi, c) \rvert_{c=c^k}$. The preferable choice of Bregman depends the cost set $\C$. We use the euclidean norm, but other choices are also possible. With stepsize $t^c_k$, the OAL cost update is
\begin{align}
    & c^{k+1} \in \argmin_c \inner*{\nabla_c   L(\pi^k, c)\rvert_{c^k}, c}  +  \frac{t^c_k}{2} \norm{c \! - \! c^k}^2\label{eq: OAL MD cost update}.
\end{align}
In the next section, we use the scheme of \Cref{alg: OAL scheme} to develop a no-regret AL algorithm.

\subsection{Convergent Online Apprenticeship Learning}\label{sec: convergence of oal}

The updates of OAL in \Cref{eq: OAL MD policy update,eq: OAL MD cost update} rely on the exact evaluation of $Q^{\pi_k,c_k}$ and $\nabla L(\pi, c; \pi^E) \rvert_{c^k}$. However, in most cases, the transition model is unknown, and therefore,
assuming access to these quantities is unrealistic.
Nevertheless, we now introduce \Cref{alg: OAL}, an OAL variant which provably minimizes the AL regret (see \Cref{def: regret}) in the tabular settings, without any restrictive assumptions. Intuitively, an AL agent should try and follow the expert's path. However, due to the environment's randomness and the possible scarcity of expert's demonstrations, it can stray afar from such path. Thus, it is crucial to explore the environment to learn a policy which keeps proximity to that of the expert (see the discussion in \Cref{sec: discussion}). To this end, \Cref{alg: OAL} uses \emph{optimistic UCB-bonuses} to explore the MDP.

The policy player has access to the costs of all state-action pairs, in each iteration. Thus, from the policy player perspective, it interacts with an adversarial MDP with full information of the costs and unknown transitions. To solve this MDP, at each iteration, \Cref{alg: OAL} improves the policy by applying an MD policy update w.r.t. a UCB-based optimistic estimation of the current $Q$-function (Line 15), relying on the techniques of \citet{cai2019provably}. The UCB bonus, $b_h^{k-1}$, added to the costs, accounts for the uncertainty in the estimation of the transitions, driving the policy to explore (Line 8).

The cost player update in the tabular setting is given by $\nabla_{c}  L(\pi^k, c)\rvert_{c=c^k} = d^E - d^{\pi^k}$, which can be evaluated using the learned model. We use costs of the form $c \in \mathcal{C}_b$, which make the cost optimization in \Cref{eq: OAL MD cost update} separable at each time-step, state and action. In this case, the euclidean distance is a natural candidate for the choice of Bregman divergence, reducing the MD update to (1) performing a gradient step towards the difference between the expert's and the agent's probability of encountering the specific state-action pair, and (2) projecting the result back to $[0,1]$ (Lines 17 and 18). 
We are now ready to state our \textbf{main theoretical result:}
\begin{restatable}{theorem}{OALconvergence}\label{theorem: convergence of OAL}
The regret of the OAL algorithm~(\Cref{alg: OAL}) satisfies with probability of $1-\delta$,
\begin{align*}
    \regret_{AL}(K) \leq  O\br*{\sqrt{H^4S^2AK}+ \sqrt{H^3 SA K^2/N}}.
\end{align*}
\end{restatable}
The regret bound in \Cref{theorem: convergence of OAL} consists of two terms. The first term shows an $O(\sqrt{K})$ rate similar to the optimal regret of solving an MDP. Perhaps surprisingly, the fact that we solve the AL problem for \emph{any possible costs} does not hurt the sample efficiency of our algorithm. The second term is a statistical error term due to the fact that we only have a limited amount of expert data. This error is independent of the AL algorithm used, and is a reminder of the fact we would like to mimic the expert itself and not the data. In this sense, it is closely related to the generalization bound in \cite{chen2020computation}, discussed in \Cref{sec: discussion}. When expert data is scarce, it could be hard to mimic the true expert's policy, and the linear error dominates the bound. Still, when the amount of experts' trajectories is of the order of the number of environment interactions, $N\propto K$, the dominant term becomes $O(\sqrt{K})$. 

Recall that previous AL results require to solve an MDP in each update, and therefore, their bounds on the amount of iterations refers to the \emph{amount of times an MDP is solved}. In stark contrast, \Cref{alg: OAL} avoids solving an MDP, and the regret bound measures the \emph{interactions with the MDP}.

In what follows we give some intuition regarding the proof of \Cref{theorem: convergence of OAL}. The full proof is found in Appendix~\ref{supp: analysis}. In the proof, we adapt the the analysis for solving repeated games using Online Optimization \cite{freund1999adaptive, abernethy2017frank} to the min-max AL problem. This allows us 
to prove the following key inequality  (\Cref{lemma: AL regret decomposition}, Appendix~\ref{supp: analysis}), which bounds the AL regret:
\begin{align*}
    \regret_{AL}(K) \! \leq   \underbrace{\vphantom{\max_c }\regret_\pi( K }_{(\romannumeral 1)} ) \! + \! \underbrace{\vphantom{\max_c }\regret_c(K}_{(\romannumeral 2)} ) \! + \! \underbrace{ 2 K \! \max_c \lvert\langle{c,  d^E \! - \! \hat d^E }\rangle \rvert }_{(\romannumeral 3)}
\end{align*}
\begin{algorithm}
\caption{Online Apprenticeship Learning (OAL)}
\label{alg: OAL}
\begin{algorithmic}[1]
\FOR{$k=1,...,K$}
    \STATE Rollout a trajectory by acting $\pi_k$
    \STATE Estimate $\hat d_k$ using the empirical model $\bar p^{k-1}$
    \STATE {\color{gray} \# Policy Evaluation}
    \STATE $\forall s \in \sset,\  V^{k}_{H+1}(s) \gets 0$
    \FOR{ $\forall h = H,..,1, s\in \sset, a\in \aset$}
        \STATE $ Q^k_h(s,a) \gets  (c_h^k  - b_h^{k-1}  +  \bar p_h^{k-1} V^{k}_{h+1})(s,a)$ 
        \STATE $ Q^k_h(s,a) \gets  \max 
            \brc{Q^k_h(s,a),0}$ 
        \STATE $V^{k}_{h}(s) \gets  \inner{Q^{k}_h(s,\cdot),\pi_h^k(\cdot \mid s)}$
    \ENDFOR       
    \STATE {\color{gray} \# Update Step}

        \FOR{$\forall h,s,a \in  [H] \times \sset\times \aset$}
        \STATE {\color{gray} \# Policy Update}
            \STATE $\pi_h^{k+1}(a\mid s)  \propto \pi_h^k(a\mid s) \exp\br*{ - t^{\pi}_k Q^{k}_h(s,a)} $
      \STATE {\color{gray} \# Costs Update}
        \STATE $c_h^{k+1}(s,a) \gets c_h^{k}(s,a) +  t_k^c \big(\hat d_h^k- \hat d_h^E\big)(s,a)$
        \STATE $c_h^{k+1}(s,a) \gets \text{Clip}\brc*{c_h^{k+1}(s,a) ,0,1}$
    \ENDFOR
    \STATE  Update counters and empirical model, $n_k,\bar p^k$
\ENDFOR
\end{algorithmic}
\end{algorithm}
Importantly, this decomposes the regret of the policy and cost players, enabling the algorithm to perform the updates in (\ref{eq: OAL MD policy update}) and (\ref{eq: OAL MD cost update}) separately. We address each of the terms:\\
\textbf{Term (\romannumeral 1). } The regret of the policy player in the adversarial MDP defined by the \emph{known} costs $\{c_k\}_{k=1}^K$. By applying the MD-based policy update (\Cref{eq: OAL MD policy update}) with an optimistic $Q$-function in \Cref{alg: OAL}, we follow the analysis in \cite{cai2019provably, efroni2020optimistic} to bound this term by $O\big(\sqrt{H^4S^2AK}\big)$ (\Cref{lemma: regret of policy player}, Appendix~\ref{supp: analysis}).\\
\textbf{Term (\romannumeral 2). } The regret of the cost player, $
\max_c \sum_{k=1}^K \langle c, d^{\pi_k,p} - \hat d^E \rangle - \sum_{k=1}^K  \langle c^k, d^{\pi_k,p} - \hat d^E \rangle$. Following MD updates in \eqref{eq: OAL MD cost update} w.r.t. the estimated occupancy measure of the current policy, this term is bounded by
$O\big(\sqrt{H^4S^2AK}\big)$ (\Cref{lemma: regret of cost player}, Appendix~\ref{supp: analysis}).\\
\textbf{Term (\romannumeral 3). } This describes the discrepancy between the expert's data and the true expert. It is bounded using Hoeffding's inequality, leading to the linear regret part of \Cref{theorem: convergence of OAL}, $O\big(\sqrt{H^3 SA K^2/N}\big)$.
\begin{figure}
    \centering
    \begin{subfigure}[t]{0.19\textwidth}
    \includegraphics[width=\textwidth]{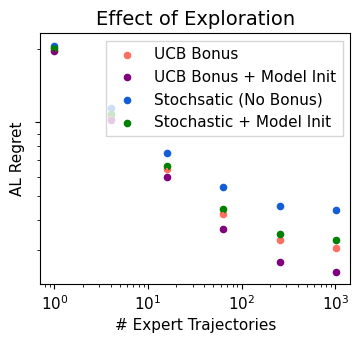} 
    \caption{}\label{fig: exploration}
    \end{subfigure}
    \begin{subfigure}[t]{0.19\textwidth}
    \includegraphics[width=\textwidth]{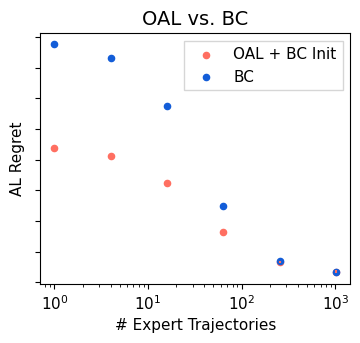} 
    \caption{}\label{fig: exploration with BC}
    \end{subfigure}    
    \caption{(a) Exploration and minimizing the AL regret.\\
    (b) Comparison between BC and AL with BC initialization.
    }
\end{figure}

\section{Discussion and Related Work}\label{sec: discussion}
\textbf{The role of exploration in AL. } A key component in our analysis for proving \Cref{theorem: convergence of OAL}, is using a UCB cost bonus to induce exploration. \emph{But should the agent explore at all, if its goal is to follow an expert?} Indeed, with infinite amount of expert trajectories, deriving a stochastic policy $\hat \pi^E$ using $\hat d^E$ (which is equivalent to BC without function approximation), allows for an exact retrieval of $\pi^E$, leading to zero regret \emph{without any exploration}. Also,
in a slightly different setting where the model is unknown and the true costs are observed,
\citet{abbeel2005exploration} provided a polynomial PAC sample complexity guarantee for imitating an expert. In contrast to our approach, they argue there is no need to encourage the algorithm to explore. Thus, intuitively, it might seem that the exploration procedure used in OAL wastefully forces the agent to explore unnecessary regions of the state-space.

However, when the number of expert trajectories is finite, directly deriving $\hat \pi$ from $\hat d^\pi$ can often lead to states unseen in the data, resulting in undetermined policies in these states which can cause an unwanted behaviour. Instead, \emph{the goal of AL is to learn policies which try to stick as close as possible to the experts trajectories}, even when unobserved states are encountered. Still, when the expert trajectories are abundant, the model can be accurately estimated in states that the expert can visit, mitigating the need for exploration. Indeed, the guarantee in \cite{abbeel2005exploration} (only) holds in this regime, when the number of trajectories is of the order of $O(1/\epsilon^3)$, where $\epsilon$ is the acceptable error.

In stark contrast, our bounds for AL do not rely on any assumption on the number of expert trajectories. As a result, our work suggests that \emph{when expert trajectories are scarce, one do need exploration to learn the transition model efficiently}, due to the fact that the model cannot be accurately estimated even in states that the expert policy might reach. Note that even in the regime where expert data is abundant, the algorithm in \cite{abbeel2005exploration} requires $O(1/\epsilon^5)$ MDP interactions to converge, which is roughly comparable to $O(K^{4/5})$ regret. In this sense, our algorithm achieves $O(\sqrt{K})$ regret, and hence requires much less interactions, in the more intricate AL setting and for any amount of trajectories.

To further address this discussion, we empirically tested the necessity of exploration for different amounts of experts' trajectories, by running \Cref{alg: OAL} with and without UCB bonuses in a tabular MDP. A fixed amount of episodes was used in all runs. For any number of trajectories, the plot was averaged over $400$ seeds, and the error bars represent $95\%$ confidence intervals. The full experimental details are found in \Cref{supp: experimental detail}. The results in \Cref{fig: exploration} show that the AL regret is consistently lower when using optimistic bonuses. This also holds when initializing the transition model in OAL by estimating it using the expert trajectories. This suggests that exploration is crucial to optimize the online performance of OAL. Moreover, \Cref{fig: exploration} shows that when more expert trajectories are available, the regret decreases down to a fixed value corresponding to the second term in \Cref{theorem: convergence of OAL}.

Finally, \citet{zhang2020generative} elegantly proved PAC convergence for an algorithm that resembles the updates in \Cref{eq: OAL MD policy update,eq: OAL MD cost update}, using neural networks for approximation. In their work, they assume bounded Radon-Nikodym derivatives, which is similar to having finite concentrability coefficients \cite{kakade2002approximately}. This bypasses the need to explore by assuming the agent policies can always reach any state that the expert reaches. 
By employing proper exploration, we refrained from such an assumption when proving the regret bound in \Cref{theorem: convergence of OAL} for the simpler tabular case. 

\textbf{On the generalization of OAL. } \citet{chen2020computation} analyzed the generalization of an AL-like algorithm in the average cost setting, which is defined as the gap between how the learned policy performs w.r.t. the true expert and its performance w.r.t the expert demonstrations. They show the generalization error depends on $O(\sqrt{(\log \mathcal{N})/N})$, where $\mathcal{N}$ is the covering number of the cost class. Specifically, when using $\C_b$ as in \Cref{alg: OAL}, this becomes $O(\sqrt{HSA/N})$, which matches the dependence on $SA/N$ in the linear term of \Cref{theorem: convergence of OAL}. This result implies that even in the tabular case, different cost classes can lead to improved generalization. However, in cost classes other than $\C_b$, the projection required to solve \Cref{eq: OAL MD cost update} can be much harder. Still, it is valuable to understand how different cost classes affects the performance of OAL when the expert trajectories are limited.

\textbf{Differences from BC.} 
Instead of minimizing the value difference directly, BC algorithms directly minimize the zero-one or $\ell_1$ loss between the agent and expert policies. A main caveat of this approach is that it fails to treat states unobserved by the expert. In this case, an $\epsilon$ error can lead to a value difference of $H^2\epsilon$ \cite{ross2010efficient}. This can be improved by further assumptions: \citet{ross2010efficient} assume that one can query the expert online;  \citet{brantley2019disagreement} add an external mechanism to attract the agent towards the expert state-action distribution, and assume the agent is concentrated around the expert. Concretely, \citet{rajaraman2020toward} shows that additional knowledge of the transition model is required to avoid this compounding error. Instead, learning a policy close to the state-action distribution of the expert is a core built-in mechanism in AL algorithms, exploiting the transition model to optimize the value-difference directly for any possible cost. This prevents the unwanted behaviour experienced in BC algorithms without relying on any additional assumptions.
Another difference between the two approaches, is that analyses of BC algorithms focus on the performance difference between the final learned policy and the optimal policy \cite{ross2010efficient}. Yet, this does not capture the online performance of the algorithm. Instead, our work uses the more common form of regret which measures the difference between the \emph{online performance of the agent} and the best policy.
Notably, many BC algorithms do interact with the environment online (e.g. \cite{ross2010efficient,brantley2019disagreement}), and it could be useful to analyze them with a similar criterion.

Finally, note that the two paradigms can be used in a complementary fashion. We ran an experiment to compare BC and \Cref{alg: OAL} with BC initialization (See \Cref{supp: experimental detail} for details). The results in \Cref{fig: exploration with BC} show that using OAL together with BC leads to an improved online regret. Yet, when expert trajectories are abundant, BC is sufficient to learn the expert policy, as discussed in the beginning of this section.

\begin{figure*}
    \centering
    \includegraphics[width=0.83\textwidth]{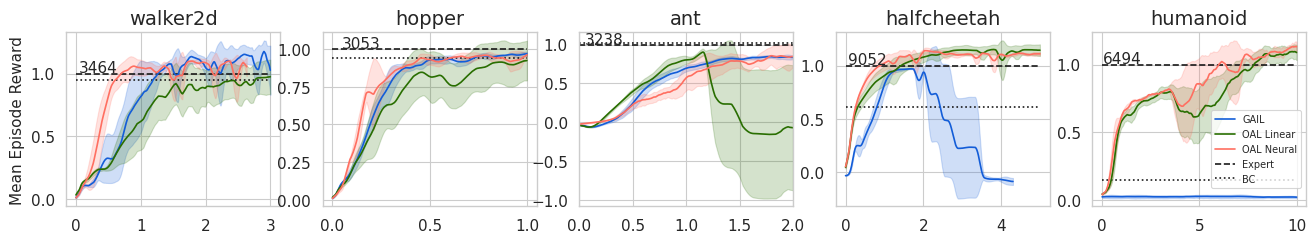} 
    \includegraphics[width=0.83\textwidth]{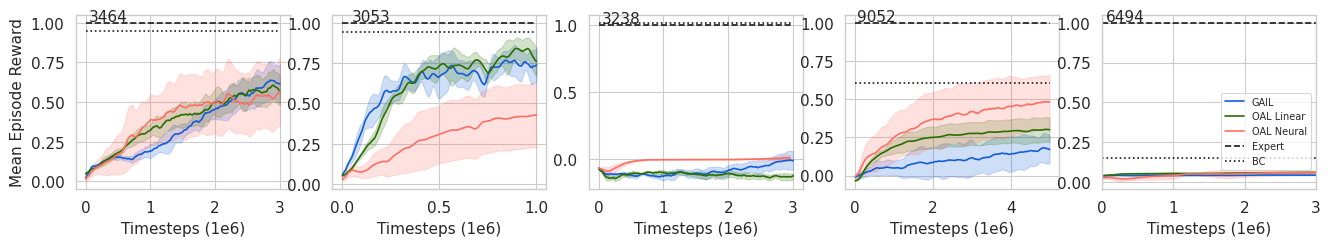} 
    \caption{OAL vs. GAIL. Policy optimizer used: (Top) off-policy MDPO; (Bottom) on-policy TRPO.}
    \label{fig:GAILvsOAL}
\end{figure*}

\section{Deep Online Apprenticeship Learning}\label{sec: deep OAL}
We now present a practical implementation of \Cref{alg: OAL scheme} using Deep RL algorithms. We implement two separate modules, a policy and a generative costs module, both of which are updated based on the MD updates in \Cref{eq: OAL MD policy update,eq: OAL MD cost update}.

For policy optimization, we use MDPO \cite{tomar2020mirror}, an MD-based off-policy deep RL algorithm. To update the policy, MDPO approximately solves \eqref{eq: OAL MD policy update} by performing several SGD steps w.r.t. its objective, keeping the target policy fixed. This enforces the stability of the policy updates required by \Cref{alg: OAL scheme}.


For the generative costs, we consider two modules. (1) A \textbf{linear costs generator} based on $\C_l$ (see \Cref{sec:al}). Using this set of costs, \Cref{eq: OAL MD cost update} can be solved in close-form. (2) A \textbf{neural network costs generator} (see $\C_n$ in \Cref{sec:al}). Note that any cost in this set must be bounded, so it can serve as a cost of an MDP. In theory, this is easily achieved by the projection step in \Cref{eq: OAL MD cost update}, which corresponds to clipping the cost to reside within the set. Yet, when using neural networks, this clipping procedure can hurt the gradient flow. Instead, we use different techniques to keep the cost bounded. First, we penalize the network's output to be close to zero, to effectively limit the size of the costs.
Second, we apply the technique proposed in \cite{gulrajani2017improved} to enforce a Lipschitz constraint on the costs. Specifically, we use a convex sum of state-actions pairs encountered by the agent's policy and the expert as an input to the costs network, and penalize the costs updates so that the gradient of the costs w.r.t. to the input would be close to $1$.
Third, we perform several gradient steps on \Cref{eq: OAL MD cost update} to force the updated costs to be close to the old ones, instead of updating the costs using gradient steps w.r.t. to the AL objective (this technique is only applied in the on-policy case discussed in \Cref{sec: experiments}). This prevents the costs from diverging too quickly. Finally, the costs given to the policy player are clipped.

\subsection{Experiments}\label{sec: experiments}
\cite{ho2016generative} pioneered the idea of solving the AL problem without solving an MDP in each step. To this end, they propose GAIL, an AL algorithm inspired by generative adversarial networks \citep[GAN;][]{goodfellow2014generative}. GAIL uses a neural network, which learns to differentiate between the policy and the expert using the GAN loss, as a \emph{surrogate} for the AL problem. In turn, the GAN loss is given as the cost of the MDP. In this section, we demonstrate that it is possible to \emph{directly solve the AL paradigm} using either \emph{a linear or a NN-based family of costs}, by following the OAL scheme.

\textbf{Experimental Setup.} We evaluated deep OAL (\Cref{sec: deep OAL}) on the MuJoCo \cite{todorov2012mujoco} set of continuous control tasks. To show the online convergence properties of AL algorithms, we present the full learning curves. We used $10$ expert trajectories in all our experiments, roughly the average amount in \cite{ho2016generative,kostrikov2018discriminator}. We tested OAL with both linear and neural costs (see \Cref{sec: deep OAL}), and compared them with GAIL. 
The same policy and cost networks were used for OAL and GAIL.

Our theoretical analysis dictates to optimize the policy using stable updates. Thus, we used two policy optimization algorithms applying the MD update: (1) on-policy TRPO, which can be seen as a hard-constraint version of MDPO \cite{shani2020adaptive}. (2) off-policy MDPO, which directly solves the policy updates in \Cref{eq: OAL MD policy update}. 

The experimental results in \Cref{fig:GAILvsOAL} show that both the linear (green) and neural (orange) versions of OAL are successful at imitating the expert. We turn to analyze the results:

\textbf{OAL vs. GAIL. } The results in \Cref{fig:GAILvsOAL} show both the linear (green) and neural (orange) versions of OAL outperform GAIL (blue), implying it is not necessary to introduce a discriminator in AL. This holds independently on the policy optimization algorithm.  Note that the performance drop of GAIL in ``Humanoid'' can be explained by the fact that \citet{ho2016generative} had to increase the amount of MDP interactions and expert trajectories in this environment.

\textbf{Neural vs. Linear. }
Surprisingly, the \emph{linear version} (green) of OAL performs almost as good as the neural one (orange). This comes with the additional benefits that linear rewards are more interpretable, they do not require to design and tune an architecture, and are faster to compute. Our results suggest that linear costs might be sufficient for solving the AL problem even in complex environment, countering the intuition and empirical results found in \cite{ho2016generative}. There, the authors argue that the main pitfall of AL is its reliance on a predetermined structured cost, which does not necessarily contains the true MDP cost. However, even if the \emph{true cost} cannot be perfectly represented by a linear function, it might still be sufficient to obtain an \emph{optimal policy}.


\textbf{MDPO vs. TRPO. } 
Inspecting \Cref{fig:GAILvsOAL}, one can see
that the off-policy MDPO version (top) significantly outperforms the on-policy TRPO (bottom) on all three algorithms. This can be attributed to two reasons: First, in our analysis, the MD policy update in \Cref{eq: OAL MD cost update} is required for efficiently solving the AL problem. MDPO is explicitly designed to optimize this policy update and therefore closer to theory. Instead, TRPO only implicitly solves this equation. Note that \citet{ho2016generative} motivated using TRPO in GAIL as preventing noisy $Q$-function estimates. Our work suggests \emph{the need for stable policy updates} as an alternative motivation.
Second, as was reported in other works \cite{kostrikov2018discriminator, blonde2019sample}, using GAIL together with an off-policy policy algorithm allows a significant boost in data efficiency. Our results strongly imply a similar conclusion.


\textbf{On Lipschitz Costs. } In \Cref{fig:lip}, we study the dependence on the Lipschitz regularization coefficient in the HalfCheetah-v3 domain. Our results implies that restricting the cost to be Lipschitz is important for OAL. Interestingly, in \cite{kostrikov2018discriminator, blonde2019sample} the authors apply the same gradient regularization technique for GAIL, even though this Lipschitz property is not necessarily required in GAIL. Indeed, \Cref{fig:lip} suggests that enforcing this regularity condition increases the stability of both GAIL and OAL. However, when used in GAIL, this technique might hurt convergence speed. Interestingly, \citet{xiao2019wasserstein} showed that solving the AL problem with Lipschitz costs is similar to GAIL with a Wasserstein distance between the occupancy measures of the agent and expert. They employ several types of regularization techniques to enforce the Lipschitz constraint and optimize the policy using TRPO. Deep OAL is different from their implementation in the following ways: (1) they focus only on on-policy scenario using TRPO; (2) they use $L_2$-regularization to decrease the costs when it is not Lipschitz, while we regularize the cost network gradients as proposed in \cite{gulrajani2017improved}.

\begin{figure}[b]
\centering
    \begin{subfigure}[t]{0.38\textwidth}
    \includegraphics[width=\textwidth]{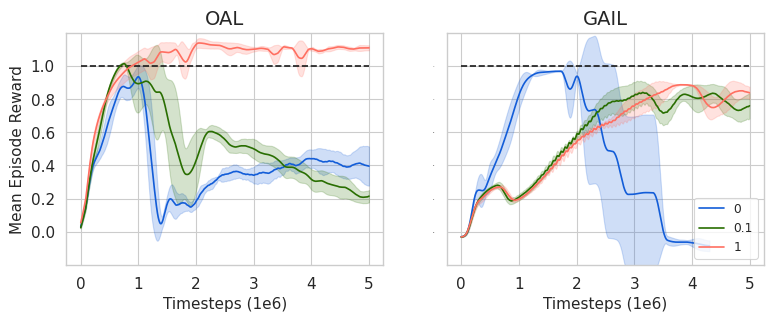} 
    \end{subfigure}
    \caption{The effect of the Lipschitz regularization.}
    \label{fig:lip}
\end{figure}

\bibliography{Bibliography}

\onecolumn
\appendix

\begin{appendices}

\include{Appendices/Analysis}

\include{Appendices/GoodEvent}
\include{Appendices/UsefulLemmas}
\include{Appendices/OALregret}
\include{Appendices/ExperimentalDetails}

\end{appendices}

\end{document}

%% file: Appendices/Analysis.tex
\section{Analysis}\label{supp: analysis}





In this section, we will prove the theoretical claims which are found in this paper. Specifically, in \Cref{supp: regret equivalence}, we prove \Cref{lemma: regrets equivalence} and show the equivalence between the AL regret (see \Cref{eq: regret of AL}) and the regret of the AL optimization problem (see \Cref{eq: AL}). Then, in \Cref{supp: OAL regret}, we provide a full proof for \cref{theorem: convergence of OAL}.

\subsection{Regret Equivalence}\label{supp: regret equivalence}

\RegretEquivalency*
\begin{proof}
\begin{align*}
    & \regret(K) \triangleq \\
    & = \max_{c\in\C} \sum_{k=1}^K \inner*{c, d^{\pi^k} - d^E}- \min_{\pi\in\ssimplex}  \max_{c\in\C}  \sum_{k=1}^K \inner*{c, d^{\pi} - d^E} \\
    & =  \max_{c\in\C} \sum_{k=1}^K \brs*{ V_1^{\pi_k, c} - V_1^{ \pi^E, c}} - \min_{\pi\in\ssimplex}  \max_{c\in\C}  \sum_{k=1}^K \brs*{ V_1^{\pi, c} - V_1^{ \pi^E, c}}\\
    & =  \max_{c\in\C} \sum_{k=1}^K \brs*{ V_1^{\pi_k, c} \! - \! V_1^{ \pi^E, c}} - \sum_{k=1}^K \min_{\pi\in\ssimplex}  \max_{c\in\C} \brs*{ V_1^{\pi, c} \! - \! V_1^{ \pi^E, c}} \\
    & =  \max_{c\in\C} \sum_{k=1}^K  \brs*{ V_1^{\pi_k, c} - V_1^{ \pi^E, c}} = \regret_{AL}(K), 
\end{align*}
where in the first transition we used the value function notation (see definition in \Cref{sec: preliminaries}), and in the last transition we used the fact that for any $\pi$, $\max_{c\in \C} V_1^{\pi, c} - V_1^{ \pi^E, c}$ is non-negative, and therefore the minimizer is $\pi= \pi^E$, for which the solution to the min-max problem is zero. 

\end{proof}

\subsection{Regret of Online Apprenticeship Learning}\label{supp: OAL regret}

In what follows, we prove \Cref{theorem: convergence of OAL}. In our proof, we deal with all probabilistic events by conditioning our analysis on the occurrence of a ``good event''. In Appendix~\ref{supp: failure events}, we define this event and bound the probability that its does not occur. For clarity and readability, we now describe the main steps of the proof before diving into the details: 
\begin{proofsketch}[Proof Sketch of \Cref{theorem: convergence of OAL}]
Our analysis relies on the following three stages:

\begin{enumerate}
    \item We prove \Cref{lemma: AL regret decomposition}, which bounds the AL regret by three independent terms: the regret of the policy player, the regret of the cost player and a statistical error term due to the finite nature of the expert samples. The fundamental relation between the the AL regret and the separate regrets of the policy and cost player is developed in \Cref{lemma: optimization term regret}. This key lemma deals with bounding the difference between the values of the sequence of learned policies and the estimated value of the expert for all possible costs. The proof of \Cref{lemma: optimization term regret} is adapted to the AL setting from the work of \cite{abernethy2017frank}, which deals with convex-concave zero-sum games.
    \item We bound each of the three terms in \Cref{lemma: AL regret decomposition} (see also \Cref{sec: convergence of oal}):
    \begin{itemize}
        \item \textbf{Policy player regret.} First, we observe that the policy player is interacting with an adversarial MDP with full information of the costs and unknown transition model. Then, by the fact we use the optimistic MD policy optimization procedure in \Cref{alg: OAL}, we can apply the results in \cite{cai2019provably,efroni2020optimistic} to bound the regret of the policy player (see \Cref{lemma: regret of policy player}).
        \item \textbf{Cost player regret.} We show that the regret of the cost player can be separated to independent MD procedures for each time-step, state and action. Then, we bound each of the problems using online MD (see \Cref{lemma: regret of cost player}).
        \item \textbf{Statistical error.} The statistical error term is bounded using Hoeffding's inequality (see event $F^V$ in Appendix~\ref{supp: failure events}).
    \end{itemize}
    \item We prove \Cref{theorem: convergence of OAL} by plugging in the above bounds in \Cref{lemma: AL regret decomposition}.
\end{enumerate}

\end{proofsketch}

We are now ready to prove \Cref{theorem: convergence of OAL}. First, we prove the following key lemma, which decompose the AL regret as described in \Cref{sec: convergence of oal}. This lemma is heavily based on \Cref{lemma: optimization term regret}, which connects the AL regret to the separate regrets suffered by the policy and cost players.

\begin{restatable}[AL Regret Decomposition]{lemma}{ALregretDecomposition}\label{lemma: AL regret decomposition}
\begin{align*}
    \regret_{AL}(K) \leq   \sum_{k=1}^K \regret_\pi( K )  +  \sum_{k=1}^K \regret_c(K)  +  2K \max_{c \in \C}  \abs*{\inner*{c,  d^E - \hat d^E}}
\end{align*}
\end{restatable}
\begin{proof}
By \eqref{eq: regret of AL},
\begin{align*}
\regret_{AL}(K) &=  \max_{c\in\C} \sum_{k=1}^K \brs*{ V_1^{\pi_k,p,c} - V_1^{ \pi^E, p,c}} \\
& = \max_{c\in\C} \sum_{k=1}^K \inner*{c, d^{\pi_k} - d^E} \\
& = \max_{c\in\C} \sum_{k=1}^K \inner*{c, d^{\pi_k} - \hat d^E + \hat d^E - d^E} \\
& \leq \max_{c\in\C} \sum_{k=1}^K \inner*{c, d^{\pi_k} - \hat d^E } + \max_{c\in\C} \sum_{k=1}^K \inner*{c, \hat d^E - d^E} \\
& \leq \max_{c\in\C} \sum_{k=1}^K \inner*{c, d^{\pi_k} - \hat d^E } + K \max_{c\in\C} \abs*{ \inner*{c, d^E - \hat d^E}} .
\end{align*}
The fourth transition is by the fact $\max_x \inner*{x,a+b}\leq \max_x \inner*{x,a} + \max_x \inner*{x,b}$.

Plugging in \Cref{lemma: optimization term regret} to bound $\max_{c\in\C} \sum_{k=1}^K \inner*{c, d^{\pi_k} - \hat d^E }$, we obtain
\begin{align*}
\regret_{AL}(K) \leq \regret_{\pi}(K) + \regret_c(K) + 2K \max_c \abs*{\inner*{  c , d^{E} -  \hat d^E}},
\end{align*}
\end{proof}

We now prove the following key lemma, which is essential to the proof of \Cref{lemma: AL regret decomposition}:

\begin{lemma}[Online Min-Max Regret Bound]\label{lemma: optimization term regret}
The following holds:
\begin{align*}
    \max_{c} \inner*{c, \sum_{k=1}^K d^{\pi_k,p} - \hat d^E} &\leq R^{\pi}_K + R^c_K + K \max_c \abs*{\inner*{  c , d^{E} -  \hat d^E}}.
\end{align*}
\begin{proof}
For brevity, we will denote $R_K^\pi\triangleq\regret_\pi(K)$ and $R_K^c\triangleq\regret_c(K)$.
By the policy optimization procedure (Lemma~\ref{lemma: regret of policy player}), we have
\begin{align*}
  \sum_{k=1}^K \inner*{c^k, d^{\pi_k,p}} \leq \min_\pi  \sum_{k=1}^K \inner*{c^k , d^{\pi,p}} + R^{\pi}_K.
\end{align*}

Dividing by $K$, we get
\begin{align}\label{eq: policy average regret}
  \frac{1}{K}\sum_{k=1}^K \inner*{c^k, d^{\pi_k,p}} & \leq \min_\pi \frac{1}{K}  \sum_{k=1}^K \inner*{c^k , d^{\pi,p}} + \frac{R^{\pi}_K}{K} \nonumber \\
 & = \min_\pi \inner*{ \frac{1}{K}  \sum_{k=1}^K c^k , d^{\pi,p}} + \frac{R^{\pi}_K}{K}
\end{align}

Similarly, from the costs optimization procedure (Lemma~\ref{lemma: regret of cost player}), we have
\begin{align*}
   \max_c \sum_{k=1}^K \inner*{c, d^{\pi_k,p} - \hat d^E} \leq \sum_{k=1}^K \inner*{c^k, d^{\pi_k,p} - \hat d^E}  + R^c_K
\end{align*}

Dividing by $K$, we get
\begin{align}\label{eq: costs average regret}
   \max_c  \inner*{c, \frac{1}{K} \sum_{k=1}^K d^{\pi_k,p} - \hat d^E}  \leq \frac{1}{K} \sum_{k=1}^K \inner*{c^k, d^{\pi_k,p} - \hat d^E}  + \frac{R^c_K}{K}
\end{align}

By combining equations \eqref{eq: policy average regret} and \eqref{eq: costs average regret},
\begin{align*}
   \max_c  \inner*{c, \frac{1}{K} \sum_{k=1}^K d^{\pi_k,p} - \hat d^E} - \min_\pi \inner*{ \frac{1}{K}  \sum_{k=1}^K c^k , d^{\pi,p}} \\
   \leq  \frac{1}{K} \sum_{k=1}^K \inner*{c^k, d^{\pi_k,p} - \hat d^E} - \frac{1}{K}\sum_{k=1}^K \inner*{c^k, d^{\pi_k,p}} +  \frac{R^{\pi}_K + R^c_k}{K}
\end{align*}

Rearranging we get
\begin{align*}
   \max_c  \inner*{c, \frac{1}{K} \sum_{k=1}^K d^{\pi_k,p} - \hat d^E} \leq \min_\pi \inner*{ \frac{1}{K}  \sum_{k=1}^K c^k , d^{\pi,p} -  \hat d^E}  +    \frac{R^{\pi}_K + R^c_k}{K}
\end{align*}

Multiplying by $K$, we have
\begin{align*}
    \max_c \inner*{c, \sum_{k=1}^K d^{\pi_k,p} - \hat d^E} &\leq \min_\pi \inner*{  \sum_{k=1}^K c^k , d^{\pi,p} -  \hat d^E} + R^{\pi}_K + R^c_k \\
    & \leq \inner*{  \sum_{k=1}^K c^k , d^{\pi^E,p} -  \hat d^E} + R^{\pi}_K + R^c_k \\
    & = \inner*{  \sum_{k=1}^K c^k , d^{E} -  \hat d^E} + R^{\pi}_K + R^c_k \\
    & = K \inner*{  \frac{1}{K}\sum_{k=1}^K c^k , d^{E} -  \hat d^E} + R^{\pi}_K + R^c_k \\
    & \leq K \max_{c} \inner*{  c , d^{E} -  \hat d^E} + R^{\pi}_K + R^c_k \\
    & \leq K \max_c \abs*{\inner*{  c , d^{E} -  \hat d^E}} + R^{\pi}_K + R^c_k,
\end{align*}
where in the fifth transition we used the fact that due to the convexity of the set $\C$, $\frac{1}{K}\sum_{k=1}^K c^k$ is always within the set.

Finally, we have that
\begin{align*}
    \max_c \inner*{c, \sum_{k=1}^K d^{\pi_k,p} - \hat d^E} &\leq R^{\pi}_K + R^c_k +  K \max_c \abs*{\inner*{  c , d^{E} -  \hat d^E}},
\end{align*}
which concludes the proof.
\end{proof}
\end{lemma}

Before we prove the main theorem for bounding the regret of OAL, we present two useful lemmas, which are proven in Sections~\ref{sec: policy optimization} and \ref{sec: costs optimization}, respectively:
\begin{restatable}[Regret of the Policy Player]{lemma}{PolicyRegret}\label{lemma: regret of policy player}
Let $t^\pi_k = \sqrt{2\log A / (H^2 K)}$. Then, conditioned on the good event, the regret of the policy player is bounded by
\begin{align*}
  \regret_\pi(K) \triangleq \max_\pi \sum_{k=1}^{K} \inner*{c^k, d^{\pi_k,p} - d^{\pi,p}} \leq O\br*{\sqrt{H^4 S^2 A K}}
\end{align*}
\end{restatable}

\begin{restatable}[Regret of the Costs Player]{lemma}{CostsRegret}\label{lemma: regret of cost player}
Conditioned on the good event,
\begin{align*}
    \regret_c(K) \triangleq \max_c \sum_{k=1}^K \inner*{c, d^{\pi_k,p} - \hat d^E} - \sum_{k=1}^K \inner*{c^k, d^{\pi_k,p} - \hat d^E} \leq O\br*{\sqrt{H^4S^2AK}}
\end{align*}
\end{restatable}

Finally, we are ready to prove \Cref{theorem: convergence of OAL}, which bounds the regret of the OAL algorithm (\Cref{alg: OAL}):
\OALconvergence*
\begin{proof}
By \Cref{lemma: AL regret decomposition},
\begin{align*}
    \regret_{AL}(K) & \leq  \regret_\pi(K) +\regret_c(K)+ 2K \max_c \abs*{\inner*{  c , d^{E} -  \hat d^E}} \\
& \leq O\br*{\sqrt{H^4S^2AK}} + O\br*{\sqrt{H^4S^2AK}} + O\br*{\sqrt{\frac{H^3 SA K^2}{N}}}\\
& = O\br*{\sqrt{H^4S^2AK} + \sqrt{\frac{H^3 SA K^2}{N}}},
\end{align*}
where we bounded the policy player and cost player regret using \Cref{lemma: regret of policy player,lemma: regret of cost player}, respectively. Finally, the last term is bounded conditioned on the good event.
\end{proof}

\subsection{Policy Optimization}\label{sec: policy optimization}

We now turn to bound the reward of the policy player in \Cref{lemma: AL regret decomposition}. As discussed in \Cref{sec: convergence of oal}, in each iteration, the policy player is allowed to observe the cost function for all time-steps, states and actions. In other words, in the perspective of the policy player, it interacts with an adversarial MDP with full information, described by the sequence of costs $\brc*{c^k}_{k=1}^K$. In order to solve this MDP efficiently, we need to address the fact that the transition model is unknown. This is achieved by adding to the costs a UCB-bonus which accounts for the uncertainty in the transition model, as done in Line~7 of \Cref{alg: OAL}. Specifically, the bonus $b_h^k(s,a)$ satisfies
\begin{align*}
    b_h^k(s,a) = \sqrt{\frac{4H^2S \log \frac{3H^2SAK}{\delta'}}{n_h^k(s,a) \vee 1}},
\end{align*}
where $a \vee b \triangleq \max\brc*{a,b}$, and $n_h^k(s,a) $ is number of times the agent has visited the state-action pair $(s,a)$ at the $h$-th time-step, until the end of the $k$-th episode. Note that $n_h^k(s,a)$ is $\F_k$-measurable. This term is needed in order to have an optimistic estimation of the bellman error of the current policy (see \citealt[Lemma~5]{efroni2020optimistic}).

Overall, our MD policy optimization procedure in \Cref{alg: OAL} exactly matches the OPPO algorithm presented in \cite{cai2019provably}. In our work, we apply this algorithm in the more specific tabular case. Thus, for readability, it is more convenient to follow the analysis and notation of the stochastic version of the POMD algorithm in \cite{efroni2020optimistic}, yet with full information of the cost. To this end, in \citep[Theorem~3.1]{cai2019provably} and \citep[Theorem~1]{efroni2020optimistic}, the authors prove the following regret bound for the policy optimization procedure in \Cref{alg: OAL},
\begin{align*}
  \regret_\pi(K) = \max_\pi \sum_{k=1}^K V_1^{\pi_k,p,c^k}(s) - V_1^{\pi,p,c^k}(s) \leq O\br*{\sqrt{H^4 S^2 A K}}.
\end{align*}
The above bound holds conditioned on the good event in Appendix~\ref{supp: failure events} (see \citealt[Appendix~B.1]{efroni2020optimistic}. Note that event $F^c$ is not required when we have full information of the cost).
Rewriting the value functions in linear form, $V_1^{\pi,p,c} = \inner*{c,d^{\pi,p}}$, the above bound translates to the following lemma:
\PolicyRegret*

\subsection{Costs optimization}\label{sec: costs optimization}

In this section, we deal with bounding the regret of the cost player in \Cref{lemma: AL regret decomposition}. The cost update in \Cref{eq: OAL MD cost update} requires to estimate the occupancy measures of the current policy for each state-action pairs. This can be done by forward recursion using the empirical model and the current policy.
Denote the occupancy measure of $d^{\pi_k,p}$ using the empirical $\bar p$ as $d^{\pi_k,\bar p}$. Thus, the gradient of the cost optimization problem at the $k$-th iteration is
\begin{align*}
\nabla_c \inner*{c, d^{\pi_k, \bar p_k} - \hat d^E} &=  d^{\pi_k, \bar p_k} - \hat d^E.
\end{align*}

Thus, the MD iterate in \Cref{eq: OAL MD cost update} becomes
\begin{align}\label{eq: costs optimization problem}
    c_{k+1} \in \argmax_ {c \in \C} \inner*{c, d^{\pi_k, \bar p_k} - \hat d^E} + \frac{1}{t_k^c} \bregman{c}{c^k}.
\end{align}

In convex optimization, the choice of Bregman divergence usually corresponds to the constraints set used. For example, in the policy optimization step, we optimized over the set of stochastic policies, and therefore used the state-wise KL divergence. However, choosing the cost set in AL is a degree of freedom, and therefore different Bregman should be chosen when different cost sets are used. For example, when considering a linear cost, $\C_l$, past works have considered different sets which led to different Bregman terms: In the projection algorithm \cite{abbeel2004apprenticeship}, the authors constrain the weights to the unit $L_2$-ball, and therefore use the euclidean distance as the Bregman term. Instead, in MWAL~\cite{syed2008game}, the authors constrain the weights to the unit simplex, and thus use the KL divergence leading to exponential updates. Although the above choices are also legitimate in the tabular case, we have chosen to focus on the unit box, $\C_b$, as the set of costs. This set is the most general bounded cost set in the tabular case, and it is typically assumed when discussing tabular MDPs. Interestingly, different choices can have an effect on the regret bounds of \Cref{lemma: AL regret decomposition}. Specifically, it can change the dependence on $H,S,A$ of both the cost player regret and the second term of \Cref{theorem: convergence of OAL}, which depends on the covering number of the cost set (see \Cref{sec: discussion}). Importantly, the set $\C_b$ is state-wise independent which makes the optimization process separable and therefore simpler. While choosing sets which enforce a global constraint on the costs is also possible, this will lead to a more complicated and less practical projection step in the solution of \Cref{eq: costs optimization problem}. 

Following the above discussion, when optimizing over the unit box, where for any $h,s,a$, $c_h(s,a)\in[0,1]$, it is most natural to use the Euclidean distance as the Bregman divergence. This leads to the following optimization problem for any $h,s,a$,
\begin{align}\label{eq: costs optimization problem pointwise}
    c_h^{k+1}(s,a) \in \argmax_ {c \in [0,1]} \br*{ d_h^{\pi_k, \bar p_k}(s,a) - \hat d_h^E(s,a)}c + \frac{1}{t_k^c}\norm{ c - c_h^k(s,a)}_2^2.
\end{align}
Then, \eqref{eq: costs optimization problem pointwise} have the following closed form projected gradient descent update,
\begin{align}\label{eq: cost update tabular}
    c_h^{k+1}(s,a) \in \text{Concat}\brc*{c^k - t_k^c (d_h^{\pi_k,\bar p_k}(s,a) - \hat d_h^E(s,a),0,1},
\end{align}
where $\text{Concat}\brc*{x,a,b}$ concatenates $x$ within $[a,b]$. Note that this simple projection step follows from the choices of the cost set and Bregman divergence.
Importantly, \Cref{eq: cost update tabular} is exactly the cost updates in lines 17,18 of \Cref{alg: OAL}.

We are now ready to prove the regret guarantee for the policy player:
\CostsRegret*
\begin{proof}
Observe that the regret for the cost optimization procedure can be decoupled in the following manner:
\begin{align}\label{eq: cost regret decoupling}
   \regret(K;c) &= \max_c \sum_{k=1}^K \inner*{c, d^{\pi_k,p} - \hat d^E} - \sum_{k=1}^K \inner*{c^k, d^{\pi_k,p} - \hat d^E} \nonumber \\
   & = \max_c \sum_{k=1}^K \inner*{c - c^k, d^{\pi_k,p} - \hat d^E}  \nonumber \\
  & = \max_c \sum_{k=1}^K \inner*{c - c^k, d^{\pi_k,p} - d^{\pi_k,\bar p_k} +  d^{\pi_k,\bar p_k} - \hat d^E} \nonumber \\
& = \max_c \sum_{k=1}^K \inner*{c - c^k, d^{\pi_k,p} -  d^{\pi_k,\bar p_k}} + \sum_{k=1}^K \inner*{c - c^k, d^{\pi_k,\bar p_k} - \hat d^E} \nonumber \\
& \leq \underbrace{\max_c \sum_{k=1}^K \inner*{c - c^k, d^{\pi_k,p} - d^{\pi_k,\bar p_k}}}_{(\romannumeral 1)} + \underbrace{\max_c \sum_{k=1}^K \inner*{c - c^k,  d^{\pi_k,\bar p_k} - \hat d^E}}_{(\romannumeral 2)} .
\end{align}
Note that the RHS in \Cref{eq: cost regret decoupling} has two additive terms. The first term is due to the statistical error in the empirical model, and will be bounded by conditioning on the good event (see Appendix~\ref{supp: failure events}). The second term will be bounded by the OMD analysis.

\paragraph{Term (i).}
For any $c$,
\begin{align*}
    \sum_{k=1}^K \inner*{c - c^k, d^{\pi_k,p} - d^{\pi_k, \bar p_k}} & = 
    \sum_{k=1}^K \sum_{h=1}^H \E \brs*{ \br*{p_h^{\pi_k}(\cdot \mid s_h^k,a_h^k) - \bar p_h^k(\cdot \mid s_h^k,a_h^k)} v_{h+1}^{\pi_k, \bar p_k, c-c^k} \mid  s_1=s, \pi_k, p} \nonumber \\
    & \leq \sum_{k=1}^K \sum_{h=1}^H \E \brs*{ \norm{p_h^{\pi_k}(\cdot \mid s_h^k,a_h^k) - \bar p_h^{\pi_k}(\cdot \mid s_h^k,a_h^k)}_1  \norm{v_{h+1}^{\pi_k, \bar p_k, c-c^k}}_\infty \mid  s_1=s, \pi_k, p }  \nonumber\\
    & \leq H \sum_{k=1}^K \sum_{h=1}^H \E \brs*{ \norm{p_h^{\pi_k}(\cdot \mid s_h^k,a_h^k) - \bar p_h^{\pi_k}(\cdot \mid s_h^k,a_h^k)}_1    \mid  s_1=s, \pi_k, p  } \\
    & \leq  C\sqrt{\ln \frac{SAHT}{\delta'}} \sum_{k=1}^K \sum_{h=1}^H \E \brs*{  H\sqrt{\frac{S}{n_h^{k-1}(s,a)\vee 1
    }} \mid s_1=s, \pi_k, p}\\
    & =  C\sqrt{S}H\sqrt{\ln \frac{SAHT}{\delta'}} \sum_{k=1}^K \sum_{h=1}^H \E \brs*{ \sqrt{\frac{1}{n_h^{k-1}(s,a)\vee 1}} \mid s_1=s, \pi_k, p}\\
    & =  CH\sqrt{S}\sqrt{\ln \frac{2SAHT}{\delta'}} \sum_{k=1}^K \sum_{h=1}^H \E \brs*{ \sqrt{\frac{1}{n_h^{k-1}(s,a)\vee 1}} \mid \mathcal{F}_{k-1}} \\
    & \leq O \br*{\sqrt{H^4 S^2 AK }},
\end{align*}
In the first transition, we used the value difference lemma (Corollary~\ref{corollary: value difference}). The second transition is by the Cauchy-Schwartz inequality. The third is by the fact that by $c-c^k \in [-1,1]$ and therefore, for any $k,h$, $\norm{v_{h+1}^{\pi_k, \bar p_k, c-c^k}}_\infty \leq H$. The fourth transition is holds by the good event for some positive constant $C$ (see Appendix~\ref{supp: failure events}). In the sixth relation we used the fact that the expectations are equivalent, since at the the policy $\pi_k$ is fully determined by the events in the filtration $\F_{k-1}$. Finally, the last transition is by applying Lemma~\ref{lemma: supp 1 factor and lograthimic factors} and excluding constant factors which do not depend on $K$.

By the fact the above inequality holds for any $c$, we get
\begin{align}\label{eq: cost regret term i good event}
    \text{Term (i)} = \max_c \sum_{k=1}^K \inner*{c - c^k, d^{\pi_k} - d^{k}} \leq O \br*{\sqrt{H^4 S^2 AK }} .
\end{align}

\paragraph{Term (ii).}
It holds that
\begin{align*}
    \text{Term (ii)} &= \max_{c\in\C_b} \sum_{k=1}^K \inner*{c - c^k,  d^{\pi_k,\bar p_k} - \hat d^E} \\
    & = \max_{c\in\C_b} \sum_{k=1}^K \sum_h \sum_{s,a} \br*{c_h(s,a) - c_h^k(s,a)}\br*{ d_h^{\pi_k,\bar p_k}(s,a) - \hat d_h^E(s,a)} \\
    & = \sum_h \sum_{s,a} \max_{c_h(s,a)\in [0,1]} \sum_{k=1}^K  \br*{c_h(s,a) - c_h^k(s,a)}\br*{ d_h^{\pi_k,\bar p_k}(s,a) - \hat d_h^E(s,a)} \\
    & \leq \sum_h \sum_{s,a} \brs*{\frac{1}{t_k^c} + \frac{t_k^c}{2}\sum_{k=1}^k \br*{d_h^{\pi_k,\bar p_k}(s,a) - \hat d_h^E(s,a)}^2 } \\
    & \leq \sum_h \sum_{s,a} \brs*{\frac{1}{t_k^c} + t_k^c \sum_{k=1}^k \br*{d_h^{\pi_k,\bar p_k}(s,a) + \hat d_h^E(s,a)} } \\
    & = \frac{HSA}{t_k^c} + t_k^c \sum_{k=1}^k \sum_h \sum_{s,a} \br*{d_h^{\pi_k,\bar p_k}(s,a) + \hat d_h^E(s,a)} \\
    & = \frac{HSA}{t_k^c} + 2t_k^c H K.
\end{align*}
The third transition is by the fact that the optimization problem can be decoupled coordinate-wise due to the structure of $\C_b$. The fourth transition is by the OMD analysis (see \Cref{lemma: euclidean OMD bound}) for any time-step, state and action. The fifth transition is due to the fact that for $a,b\in[0,1]$, $(a-b)^2\leq 2a+2b$. The last transition is due to the fact that for any occupancy measure $d$ and time-step $h$, $ \sum_{s,a} d_h(s,a)=1$.

Thus, by choosing $t_k^c = \sqrt{\frac{SA}{2K}}$
\begin{align}\label{eq: cost regret term ii OMD}
    \text{Term (ii)} \leq O\br*{\sqrt{H^2 SAK}}
\end{align}

Finally, plugging \eqref{eq: cost regret term i good event} and \eqref{eq: cost regret term ii OMD} in the the two terms in equation \eqref{eq: cost regret decoupling}, we get that conditioned on the good event,
\begin{align*}
   \regret_c(K)  \leq O\br*{\sqrt{H^4 S^2 AK }}
\end{align*}
\end{proof}

%% file: Appendices/GoodEvent.tex
\section{Failure Events}\label{supp: failure events}

Define the following failure events.
\begin{align*}
    &F^p_k=\brc*{\exists s,a,h:\ \norm{p_h(\cdot\mid s,a)- \bar{p}_h^k(\cdot\mid s,a)}_1 \geq \sqrt{\frac{4S\ln\frac{3SAHT}{\delta'}}{n_h^{k-1}(s,a)\vee 1
    }}}\\
    &F^N_k = \brc*{\exists s,a,h: n_h^{k-1}(s,a) \le \frac{1}{2} \sum_{j<k} w_j(s,a,h)-H\ln\frac{SAH}{\delta'}}\\
    &F^V = \brc*{ \exists c\in \C_b: \abs{\inner{c, d^E - \hat d^E}} \geq \sqrt{\frac{H^3 SA\log {\frac{4}{\delta'}}}{2N}}}.
\end{align*}

Furthermore, the following relations hold.

\begin{itemize}
    \item Let $F^P=\bigcup_{k=1}^K F^{p}_k.$ Then $\Pr\brc*{ F^p}\leq \delta'$, holds by \citep{weissman2003inequalities} while applying union bound on all $s,a$, and all possible values of $n_k(s,a)$ and $k$. Furthermore, for $n(s,a)=0$ the bound holds trivially. 
    \item Let $F^N=\bigcup_{k=1}^K F^N_k.$ Then, $\Pr\brc*{F^N}\leq \delta'$. The proof is given in \citep{dann2017unifying} Corollary E.4.
    \item First, note that any $c\in \C_b$ can be written as a convex sum of edges of the unit box over $H\times\sset\times \aset$. Thus, in order to bound $\abs*{\inner*{c, d^E - \hat d^E}}$ for any $c\in \C_b$, it suffices to bound this term for all the edges of the $2^{SAH}$ unit box and apply the triangle inequality.
    Now, take a some fixed edge $c_{edge}$,
    \begin{align*}
        \abs*{\inner*{c_{edge}, d^E - \hat d^E} } = \abs*{\inner*{c_{edge}, \E \hat d^E - \hat d^E}} = \abs*{\E \brs*{\inner*{c_{edge}, \hat d^E }} -  \inner*{c_{edge}, \hat d^E } } .
    \end{align*}
    By Hoeffding's inequality, we get that w.p. $\frac{\delta'}{2^{SAH}}$ it holds that $\abs*{\inner*{c_{edge}, d^E - \hat d^E}} \leq H \sqrt{\frac{ SAH\log\frac{4}{\delta'}}{2N}}$. Finally, by taking a union bound over all possible edges, we get $\Pr \{ F^V\} \leq \delta'$.
    
\end{itemize}

\begin{lemma}[Good Event]\label{lemma: ucrl failure events}
Setting $\delta'=\frac{\delta}{3}$ then $\Pr\brc{F^p\bigcup F^N \bigcup F^V}\leq \delta$. When the failure events does not hold we say the algorithm is outside the failure event, or inside the good event $G$.
\end{lemma}

%% file: Appendices/UsefulLemmas.tex
\section{Useful Lemmas}\label{sec: useful lemmas}

\subsection{Difference Lemmas}

The following lemma is taken from \cite{efroni2020optimistic}[Lemma 1] (originally adapted from the analysis of the first term, in \cite{cai2019provably}[Lemma 4.2]). It extends the value difference lemma (see Corollary~\ref{corollary: value difference}) which is widely used in the RL literature.

\begin{lemma}[Extended Value Difference]\label{lemma: extended value difference}
Let $\pi,\pi'$ be two policies, and $\M = (\sset, \aset, \brc*{p_h}_{h=1}^H, \brc*{c_h}_{h=1}^H)$ and $\M' = (\sset, \aset, \brc*{p'_h}_{h=1}^H, \brc*{c'_h}_{h=1}^H)$ be two MDPs.
Let $\hat Q_h^{\pi,\M}(s,a)$ be an approximation of the $Q$-function of policy $\pi$ on the MDP $\M$ for all $h,s,a$, and let  ${\hat V_h^{\pi,\M}(s) = \inner*{\hat Q_h^{\pi,\M}(s,\cdot) ,\pi_h(\cdot\mid s)}}$.
Then,
\begin{align*}
    & \hat V_1^{\pi,\M}(s_1) - V_1^{\pi',\M'}(s_1)=
    \\
    &\!
    \sum_{h=1}^H \! \E \brs*{ \inner*{\hat Q_h^{\pi,\M}(s_h,\cdot), \pi'_h(\cdot \mid s_h) - \pi_h(\cdot 
    \mid s_h)} \mid s_1,\pi',p'}+
    \\
    &   
    \!\sum_{h=1}^H \!\E \brs*{\hat Q_h^{\pi\!,\M}\!(s_h,\!a_h) \! -\! c_h' \!(s_h,a_h)\! -\! p'_h\!(\cdot | s_h,a_h) \hat V_{h+1}^{\pi\!,\M}\!\mid \! s_1,\!\pi',\! p'}
\end{align*}
where $V_1^{\pi',\M'}$ is the value function of $\pi'$ in the MDP $\M'$.
\end{lemma}

By replacing the approximation in the last lemma with the real expected value, we get the following well known result: 

\begin{corollary}[Value difference]\label{corollary: value difference}
Let $\M,\M'$ be any $H$-finite horizon MDP.
Then, for any two policies $\pi,\pi'$, the following holds
\begin{align*}
    & V_1^{\pi,\M}(s) - V_1^{\pi',\M'}(s)
    =
    \\
    & \quad =
    \sum_{h=1}^H \E \brs*{ \inner*{ Q_h^{\pi, \M\phantom{'}}(s_h,\cdot), \pi_h(\cdot \mid s_h )-\pi'_h(\cdot \mid s_h )} \mid s_1 = s,\pi',\M'}
    \\
    & \quad + 
    \sum_{h=1}^H \E \brs*{ \br*{c_h(s_h,a_h) - c'_h(s_h,a_h)} + \br*{p_h(\cdot \mid s_h,a_h)- p'_h(\cdot \mid s_h,a_h)} V_{h+1}^{\pi,\M} \mid s_h=s,\pi', \M'} .
\end{align*}
\end{corollary}

\subsection{Online Mirror Descent}

In each iteration of Online Mirror Descent (OMD), the following problem is solved:

\begin{align}\label{eq: OMD iterates}
x_{k+1} \in \argmin_{x\in \Delta_d}  \inner*{g_k, x - x_k } + \frac{1}{t_k}\bregman{x}{x_k}.
\end{align}

The following lemma describes the regret of a general OMD procedure, for a general convex constraint set and Bregman divergence:
\begin{lemma}[OMD Regret Bound, \citealt{orabona2019modern}, Theorem~6.8]
    Let $B_\omega$ the Bregman divergence w.r.t. $\psi: X \rightarrow \reals$ and assume $\omega$ to be $\lambda$-strongly convex with respect to $\norm{\cdot}$ in $V$. Let $V\subseteq X$ a non-empty closed convex set. Also, assume $V \subseteq \text{int}X$. Set $g_t\in \partial f_t(x_t)$. Set $x_1\in \mathcal{X}$ such that $\psi$ is differentiable in $x_1$. Assume $t_k = t_K$, for $k=1,...,K$. Then, $\forall u \in V$, the following regret bound holds:
    \begin{align*}
        \sum_{k=1}^K \br*{f_t(x_t) - f_t(u)} \leq \frac{\bregman{u}{x_1}}{t_K} + \frac{t_K}{2\lambda}\sum_{k=1}^K\norm{g_t}_\star^2
    \end{align*}
\end{lemma}

In our analysis (see \Cref{lemma: regret of cost player}), we apply the above lemma for the case when the constraint set is $x\in[0,1]$ and the euclidean distance is chosen as the Bregman divergence. Note that in this case, for any $u\in X$, $\bregman{u}{x_1}\leq 1$. Overall, this results in the following corollary:
\begin{corollary}[OMD Euclidean Regret Bound]\label{lemma: euclidean OMD bound}
    Let $\omega=\norm{\cdot}_2^2$ such that $\bregman{x}{y}=\norm{x-y}_2^2$, and $\omega$ is $1$-strongly convex w.r.t. $\norm{\cdot}_2$. Let $X = \brc*{x \mid x\in [0,1]} $. Set $g_t\in \partial f_t(x_t)$. Set an arbitrary $x_1\in X$. Assume $t_k = t_K$, for $k=1,...,K$. Then, $\forall u \in V$, the following regret bound holds:
    \begin{align*}
        \sum_{k=1}^K \br*{f_t(x_t) - f_t(u)} \leq \frac{1}{t_K} + \frac{t_K}{2}\sum_{k=1}^K\norm{g_t}_2^2
    \end{align*}
\end{corollary}

\subsection{Bounds on the Visitation Counts}

\begin{lemma}[e.g.\ \citealt{zanette2019tighter}, Lemma 13]
\label{lemma: supp 1 1/N  factor and lograthimic factors}
Outside the failure event, it holds that
\begin{align*}
    \sum_{k=1}^K \sum_{h=1}^H \E\brs*{ {\frac{1}{n_{k-1}(s_h^k,\pi_k(s_h^k))\vee 1}} \mid \F_{k-1} }\leq \tilde{O} \br*{SAH^2}.
    \end{align*}
\end{lemma}

\begin{lemma}[e.g.\ \citealt{efroni2019tight}, Lemma 38]
\label{lemma: supp 1 factor and lograthimic factors}
Outside the failure event, it holds that
$$\sum_{k=1}^K\sum_{h=1}^H \E\brs*{ \sqrt{\frac{1}{n_{k-1}(s_h^k,\pi_k(s_h^k))\vee 1}} \mid \F_{k-1} }\leq \Tilde{O}\br*{\sqrt{SAH^2K} +SAH}.$$
\end{lemma}

In both~\citealt{zanette2019tighter,efroni2019tight}, these results were derived for MDPs with stationary dynamics. Repeating their analysis, in our case, an additional $H$ factor emerges as we consider MDPs with non-stationary dynamics.

%% file: Appendices/OALregret.tex
\section{A Brief Discussion on Tabular Apprenticeship Learning}\label{supp: AL discussion}

\paragraph{Linear cost and known model assumptions.}
As mentioned before, a typical assumption in AL is that the cost, which is generally a vector in $\mathbb{R}^S$, can be represented in a lower dimension $\mathbb{R}^d$ as a linear combination of the observed features. A second assumption that is typically made in AL, is that the model is either known or given via a perfect simulator. These assumptions allowed the AL literature to present the sample complexity results as a function of the feature dimension $d$, with no dependence on the size of the state-action space, and without making further assumptions regarding the dynamics (e.g., as in linear MDPs \citep{jin2020provably}). When the model has to be estimated from samples, the agent will have to explore, and therefore the sample complexity results will depend on $S$ unless we further make assumptions regarding the dynamics (e.g., \cite{abbeel2005exploration}). 

\paragraph{The AL problem as a regularizer.}
AL usually follows two assumption to regularize the policy class: (1) the expert is minimizing some cost, and (2) this cost is within a small set. Note here that even without (2) (e.g., in tabular MDPs), (1) is regularizing the policy class. This is because the set of policies that minimize an MDP is typically smaller than the set of all (deterministic) policies.

%% file: Appendices/ExperimentalDetails.tex
\section{Experimental Details}\label{supp: experimental detail}
\subsection{Tabular OAL with Exploration}

In \Cref{fig: MDP}, we describe the finite-horizon MDP which is used in our tabular experiments (e.g., \Cref{fig: exploration}). Specifically, we used an horizon of $H=32$. This stochastic chain MDP consists of two possible states, $s_0$ and $s_1$, at any time-step. In each state, the agent faces two choices, $a_0$ and $a_1$. While in $s_0$, by performing $a_0$, the agent remains in the same state at the next time-step w.p. $1-\alpha$, or otherwise transitions to $s_1$; instead, by performing $a_1$ the agent deterministically transitions to $s_1$. In $s_1$, both actions lead to $s_1$ deterministically.

In reward-based RL, when a big sparse reward is given in $s_0$ (only) in the \emph{last time-step}, and a very small reward is given when the agent encounters $s_1$ in \emph{any time-step}, this MDP is considered a hard task which requires exploration (see similar example in \cite{osband2016deep}). When the reward in state $s_0$ at the end of the chain is big enough, the optimal policy is to always choose $a_0$. Following this reasoning, in our experiments, we sampled expert trajectories from an expert policy which always chooses action $a_0$. Notably, the sparsity of the reward in the optimal policy requires the agent to perform many actions that do not lead to immediate rewards. Thus, the RL agent must explore to prevent the agent from converging to a sub-optimal policy. Differently, in AL, when entire expert trajectories are given to the agent in advance, the agent receives a cost in \emph{all} state-action pairs which the expert visited. Therefore, even in exploratory environments like the stochastic chain, the costs given by OAL are not necessarily sparse. An instant question is whether the same transition model still requires exploration when solved using the costs of an AL agent.

Indeed, as discussed in \Cref{sec: discussion}, \Cref{fig: exploration} suggests that exploration is required to attain low regret for this MDP when using OAL. To further understand the effect of the chain MDP structure on the AL problem, we tested OAL with different values of the transition stochasticity parameter, $\alpha$. Similarly to \Cref{fig: exploration}, we examined how the OAL regret is affected by two factors: using UCB exploration and initializing the learned transition model with the expert trajectories (similarly to \cite{abbeel2005exploration}). The results are reported in \Cref{fig: noise experiment}. We ran all seeds for $K=10000$ episodes. We now turn to discuss the results:

\textbf{On the effect of exploration.} \Cref{fig: noise experiment} shows that for any value of $\alpha$, using the UCB-bonus exploration improves the overall regret, fortifying the results in \Cref{fig: exploration}.

\textbf{Initializing the transition model with the experts' trajectories.} Interestingly, when $\alpha$ is large, this procedure does not improve the regret by much. This is due to the fact that in this case the model is poorly estimated at $s_0$ in later time-steps, and there is no escape from exploring the MDP to improve the model estimation. Instead, when $\alpha$ is small, the expert trajectories can provide a good estimate for the transition model in states the expert could visit, leading to a significant boost in performance which is almost equal to the one attained when performing exploration. Still, using this technique is orthogonal to applying exploration, and using the two techniques together greatly improves the performance in this case.

\begin{figure}[t]
	\centering
		\begin{tikzpicture}[->,>=stealth',shorten >=1pt,auto,node distance=2.6cm,
		semithick, state/.style={circle, draw, minimum size=1.1cm}]
		\tikzstyle{every state}=[thick]
		]
		
		\node[state] (S00) {\large $s_0$};
		\node[state] (S10) [right of=S00] {\large $s_0$};
		\node[state] (S11) [below of=S10] {\large $s_1$};
		\node[state] (S20) [right of=S10] {\large $s_0$};
		\node[state] (S21) [below of=S20] {\large $s_1$};
		\node[state] (S30) [right of=S20] {\large $s_0$};
		\node[state] (S31) [below of=S30] {\large $s_1$};
		\node[state,draw=none] (H0) [above of=S00, node distance=1.3cm] {\large $h=0$};
		\node[state,draw=none] (H1) [above of=S10, node distance=1.3cm] {\large $h=1$};
		\node[state,draw=none] (H2) [above of=S20, node distance=1.3cm] {\large $h=2$};
		\node[state,draw=none] (H3) [above of=S30, node distance=1.3cm] {\large $h=H-1$};
		\path[color=blue]
		(S00) edge node[pos=0.1,above]{ } node [above] {\large $1-\alpha$} (S10)
		(S00) edge [bend left] node[pos=0.1,below]{ }         node [below] {\large $\alpha$} (S11)
		(S10) edge node[pos=0.1,above]{ } node [above] {\large $1-\alpha$} (S20)
		(S10) edge [bend left] node[pos=0.1,below]{ }         node {\large $\alpha$} (S21);
		\path[color=red]
		(S00) edge [bend right] node[pos=0.1,below]{ }         node [below] {\large $1$} (S11)
		(S11) edge node[above]{ }         node [below] {\large $1$} (S21)
		(S10) edge [bend right] node[pos=0.1,below]{ }         node [below] {\large $1$} (S21);
		\path[dashed]
		(S20) edge node[pos=0.1,above]{ } node [above] {} (S30)
		(S20) edge [bend left] node[pos=0.1,above]{ } node [above] {} (S31)
		(S20) edge [bend right] node[pos=0.1,above]{ } node [above] {} (S31)
		(S21) edge node[pos=0.1,above]{ } node [above] {} (S31);
		
		\end{tikzpicture}
	\caption{A finite horizon chain environment. Actions $a_0$ and $a_1$ are in blue and red, respectively.}
	\label{fig: MDP}
\end{figure}
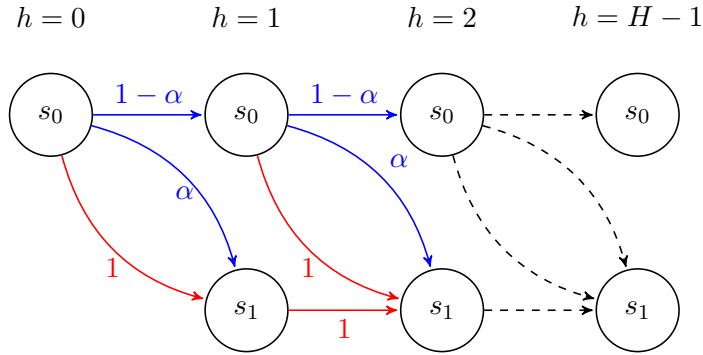

\subsection{Tabular OAL with BC Initialization}
\begin{figure*}[h]
    \centering
    \includegraphics[width=0.5\textwidth]{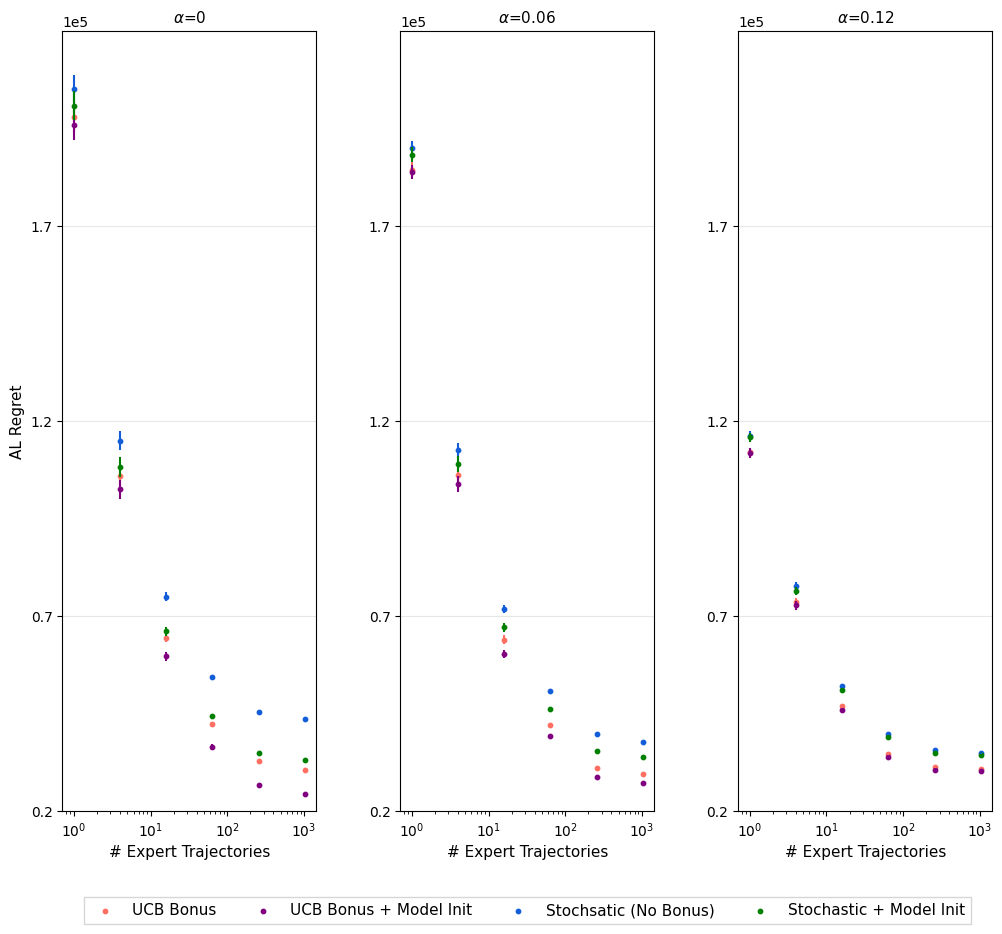} 
    \caption{The effect of exploration and expert model initialization on the OAL regret for different values of the model stochasticity parameter, $\alpha$.}
    \label{fig: noise experiment}
\end{figure*}
In \Cref{fig: MDP BC}, we describe the MDP used in the experiment in \Cref{fig: exploration with BC}. At the beginning of any episode, the agent is randomly spawned at any of the $50$ states. The Expert policy is to always reach state $s_1$ at the last time step. I.e., the expert policy plays $a_0$ (blue) at any state. We used $1000$ seeds for BC and $10$ seeds for AL. Finally, the $95\%$ confidence intervals reside within the plotted dots in \Cref{fig: exploration with BC}.

Notably, imitating the expert using BC will lead to a uniformly random policy in starting states that are unobserved in the data. In turn, when the amount of expert trajectory is small (and particularly, when it is smaller than the number of starting states), BC would not be capable to accurately learn the expert policy, and would therefore suffer linear regret in episodes in which the starting state is not observed in the expert data. Instead, the AL paradigm bypasses this issue by learning a policy with state-action occupancy measure which is close to that of the expert. Then, because for any possible trajectory the expert is always at state $s_1$ when $h=1$, the AL agent will learn to always play $a_0$, recovering the expert's behaviour. Still, initializing the AL agent with BC, allows to enjoy the best of both world: 1) the offline BC procedure warm-starts the algorithm, reducing OAL regret by not starting with a totally random policy; 2) then, the AL learning paradigm allows to fully recover the expert policy, further reducing the OAL regret.
The results in \Cref{fig: exploration with BC} captures this behaviour: when the amount of expert trajectories is small, the regret of using OAL is much smaller than using only BC. This difference deceases as the number of trajectories gets bigger, and is almost unnoticed when most starting states are observed in the expert data.
\begin{figure}[h]
	\centering
		\begin{tikzpicture}[->,>=stealth',shorten >=1pt,auto,node distance=2cm,
		semithick, state/.style={circle, draw, minimum size=1.1cm}]
		\tikzstyle{every state}=[thick]
		]
		
		\node[state] (S00) {\large $s_1$};
		\node[state] (S01) [below of=S00] {\large $s_2$};
		\node[state] (S02) [below of=S01] {\large $s_{50}$};
        \node[state] (S10) [right of=S00] {\large $s_1$};
		\node[state] (S11) [below of=S10] {\large $s_2$};
		\node[state] (S12) [below of=S11] {\large $s_{50}$};
        \node[state] (S20) [right of=S10] {\large $s_1$};
		\node[state] (S21) [below of=S20] {\large $s_2$};
		\node[state] (S22) [below of=S21] {\large $s_{50}$};
		\node[state,draw=none] (H0) [above of=S00, node distance=1.3cm] {\large $h=0$};
		\node[state,draw=none] (H1) [above of=S10, node distance=1.3cm] {\large $h=1$};
		\node[state,draw=none] (H2) [above of=S20, node distance=1.3cm] {\large $h=2$};
		\path[color=blue]
		(S00) edge [bend left] node[pos=0.1,above]{ } node [above] { } (S10)
		(S01) edge node[pos=0.1,below]{ } node [below] { } (S10)
		(S02) edge node[pos=0.1,below]{ } node [below] { } (S10)
		(S10) edge [bend left] node[pos=0.1,above]{ } node [above] { } (S20)
		(S11) edge [bend left] node[pos=0.1,above]{ } node [above] { } (S21)
		(S12) edge [bend left] node[pos=0.1,above]{ } node [above] { } (S22);
		\path[color=red]
		(S00) edge [bend right] node[pos=0.1,below]{ } node [below] { } (S10) 
		(S01) edge [bend right] node[pos=0.1,below]{ } node [below] { } (S11)
		(S02) edge [bend right] node[pos=0.1,below]{ } node [below] { }  (S12)
		(S10) edge [bend right] node[pos=0.1,below]{ } node [below] { } (S20)
		(S11) edge [bend right] node[pos=0.1,below]{ } node [below] { } (S21)
		(S12) edge [bend right] node[pos=0.1,below]{ } node [below] { } (S22);

		\path (S01) -- (S02) node [black, font=\Huge, midway, sloped, left, pos=1.25] {$\dots$};
		\path (S11) -- (S12) node [black, font=\Huge, midway, sloped, left, pos=1.25] {$\dots$};
		\path (S21) -- (S22) node [black, font=\Huge, midway, sloped, left, pos=1.25] {$\dots$};
		
		\end{tikzpicture}
	\caption{The MDP used for the comparison between BC and AL in \Cref{fig: exploration with BC}. Actions $a_0$ and $a_1$ are in blue and red, respectively.}
	\label{fig: MDP BC}
\end{figure}
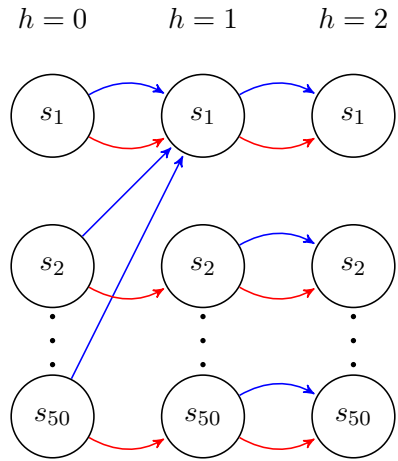

\clearpage
\subsection{Deep OAL}

For each experiment, we averaged the results over $5$ seeds and plotted $95\%$ confidence intervals. We used $10$ expert trajectories in all our experiments, roughly the average amount in \cite{ho2016generative,kostrikov2018discriminator}. We verified this choice by training OAL on ``Walker2d'' with different number of trajectories, as reported in \Cref{fig: trajectories amount}. Similarly to \cite{ho2016generative,kostrikov2018discriminator}, our results suggest that OAL performs reasonably well regardless on the amount of expert trajectories. The dashed line in the figures represents the average performance of the expert.
We used the Stable-Baselines \cite{stable-baselines} code-base to reproduce GAIL and to implement our algorithms, and did not change the default implementation parameters. The same hyper-parameters were utilized for all domains.
\begin{figure*}
    \centering
    \includegraphics[width=0.25\textwidth]{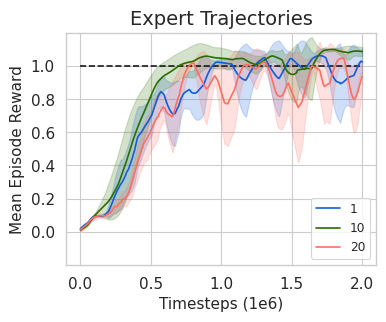} 
    \caption{
    OAL performance depending the amount of expert trajectories.}\label{fig: trajectories amount}
\end{figure*}

In \Cref{tab:Hyperparams-on,tab:Hyperparams-off}, we specify the hyper-parameters used in the implementation of our algorithms. All other hyper-parameters are set to their defaults in \cite{stable-baselines}. In what follows, we explain the meaning of the main hyper-parameters intoduced in our algorithms. The ``cost update frequency'' hyper-parameter determines the amount of interaction with the MDP after which the costs are updated using \Cref{eq: OAL MD cost update}. The ``cost range'' hyper-parameter describes how the costs are normalized or concatenated, when given to the policy player. As discussed in \Cref{sec: deep OAL,sec: experiments}, in the neural version of OAL, we penalize the costs to be globally Lipschitz. The amount of regularization is defined by the ``Liphscitz coefficient'' hyper-parameter. Notably, this coefficient is the one used when describing the effect of Lipschitz regularization in \Cref{fig:lip}.
In the on-policy neural version of OAL, we solve \Cref{eq: OAL MD cost update} by performing several gradient descent objective. To this end, the ``cost Bregman coefficient'' hyper-parameter is the coefficient used for the Bregman term in \Cref{eq: OAL MD cost update}, and the ``cost Bregman coefficient'' is related to the inverse of the ``cost step size'', $t^c$. ``cost MD sgd updates'' describes the number of gradient descent iterations. The pseudo-code for deep OAL is described in \Cref{alg:Neural Off-policy OAL}. Finally, \citet{kostrikov2018discriminator} showed that running GAIL using TRPO on complex domains such as ``HalfCheetah'', requires as many as 25 Million interactions with the environment. Therefore, to save computation when comparing the performance of OAL and GAIL using TRPO as optimizer, we used a pretraining procedure for the more complex environments, ``HalfCheetah'' and ``Humanoid'', as done in the benchmarks in \cite{baselines}. Specifically, in the TRPO (on-policy) versions of GAIL and OAL which use a neural cost network, we pretrained the policy network using 1000 epochs of BC w.r.t. to the expert trajectories. In \Cref{fig: BC experiment}, we compared the performance of GAIL and OAL with and without pretraining. The results show that pretraining improves the convergence on both OAL and GAIL. 
Still, the results suggest that BC initialization has a slightly greater effect on OAL than on GAIL.
Interestingly, the linear version of OAL is still competitive, even though it does not use pretraining. 
For the full implementation, see the \href{https://anonymous.4open.science/r/OAL/README.md}{\color{blue} OAL github repository}.
\begin{figure*}[b]
    \centering
    \includegraphics[width=0.5\textwidth]{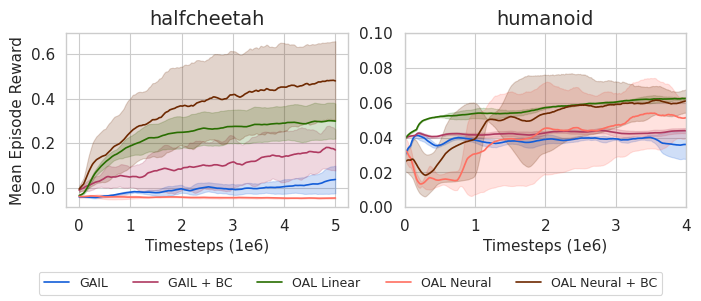} 
    \caption{Behavioural cloning initialization in on-policy OAL.}
    \label{fig: BC experiment}
\end{figure*}
\clearpage
\subsubsection{Pseudocode}
\phantom{2}

\label{appsec: pseudocodes}
\begin{algorithm}
\begin{small}
\caption{Deep OAL with Off-Policy MDPO}
\label{alg:Neural Off-policy OAL}
\begin{algorithmic}[1]
\STATE {\bfseries Initialize} Replay buffer $\mathcal{D}=\emptyset$; Value networks $V_\phi$ and $Q_\psi$; \\ Policy networks $\pi_{\text{new}}$ and $\pi_{\text{old}}$ with parameters $\theta$;  \\ Trajectory Replay buffer $\mathcal{D}_\pi=\emptyset$; Cost Network $c$ with parameters $\omega$.
\FOR{$k =1,\ldots,K$}
    \STATE Take action $a_k\sim\pi_{\theta_k}(\cdot|s_k)$, observe $s_{k+1}$
    \STATE Add $(s_k,a_k, s_{k+1})$ to the policy replay buffer $\mathcal{D}$;
    \STATE Add $(s^{\pi}_k,a^{\pi}_k)$ to the trajectory replay buffer $\mathcal{D}_\pi$;
    \STATE {\color{brown}\# Policy Update}
    \STATE Sample a batch $\{(s_j, a_j, s_{j+1})\}_{j=1}^N$ from $\mathcal{D}$
        \STATE Generate $c_{\phi_k}(s_j,a_j)$ using the cost network
    \STATE {\color{gray}\# Policy Improvement $\;$ {\em (Actor Update)}}
    \STATE Set the new policy $\theta_{k+1}$ by performing $m$ SGD updates on the MDPO objective $L_{\pi}(\theta, \theta_{k})$ {\color{gray} \# see \cite{tomar2020mirror}}
    \STATE {\color{gray}\# Policy Evaluation $\;$ {\em (Critic Update)}}
    \STATE Update $\phi$ and $\psi$ by minimizing the loss functions \\ $L_{V_\phi} = \frac{1}{N} \sum_{j=1}^N \big[V_\phi(s_j) - Q_\psi\big(s_j,\pi_{\theta_{k+1}}(s_j)\big)\big]^2$; \\ $L_{Q_\psi} = \frac{1}{N} \sum_{j=1}^N \big[c_{\phi_k}(s_j,a_j) + \gamma V_\phi(s_{j+1}) - Q_\psi(s_j,a_j)\big]^2$
    \STATE {\color{brown} \# Cost Update}
    \IF{$s_{k+1}$ signals the end of trajectory}
    \STATE Sample an expert trajectory $\mathcal{D}_E =  \{ s_j^E,a_j^E\} $
    \STATE Update $\omega$ by minimizing the AL objective w.r.t. to the costs, using $\mathcal{D}_\pi, \mathcal{D}_E$ {\color{gray} \# see \Cref{eq: OAL MD cost update}/\Cref{sec: deep OAL}}
    \STATE $\mathcal{D}_c = \emptyset$
    \ENDIF
\ENDFOR
\end{algorithmic}
\end{small}    
\end{algorithm}
\clearpage

\subsubsection{Hyperparameters}
\phantom{2}
\label{appsec: hyperparameters}
\begin{table}[hb]\label{table: hyperparameters on-policy}
\begin{center}
\scalebox{0.8}{%
\begin{tabular}{ c | c  c  c} \toprule
    Hyperparameter & GAIL & OAL Linear & OAL Neural \\ \hline 
    update frequency (T) & 1024 & 2000 & 2000 \\
    cost update frequency (T) & 3072 & 6000 & 6000 \\
    cost Adam step size & $3 \times 10^{-4}$ & - & $9 \times 10^{-5}$  \\
    cost Bregman coefficient & - & - & 100 \\
    cost MD sgd updates& 1 & 1 & 10 \\
    cost range & - & $L_2$ normalized & [-10,10]
    \\
    cost hidden layers & 3 & - & 3 \\
    cost activation & $\tanh$ & - & $\tanh$ \\
    cost hidden size & 100 & - & 100 \\
    cost step size $t^c$ & - & 0.05 & - \\
    Lipschits coefficient & 0 & - & 1.0 \\
    \hline
    policy stepsize $t^\pi$  & \multicolumn{3}{|c}{0.5} \\
    Adam step size & \multicolumn{3}{|c}{$3\times 10^{-4}$} \\
    entropy coefficient & \multicolumn{3}{|c}{0.0} \\
    cost entropy coefficient & \multicolumn{3}{|c}{0.001} \\
    discount factor & \multicolumn{3}{|c}{0.99} \\
    minibatch size & \multicolumn{3}{|c}{128} \\
    \#runs used for plot averages & \multicolumn{3}{|c}{5} \\
    confidence interval for plot runs & \multicolumn{3}{|c}{$\sim$ 95\%} \\
    \bottomrule
\end{tabular}}
\end{center}
\caption{Hyper-parameters of all on-policy methods.}
\label{tab:Hyperparams-on}
\end{table}
\vspace{-1ex}
\begin{table}[hb]\label{table: hyperparameters off-policy}
\begin{center}
\scalebox{0.8}{%
\begin{tabular}{ c | c  c  c  c  c  c} \toprule
    Hyperparameter & GAIL & OAL Linear & OAL Neural \\ \hline 
    cost range & - & $L_2$ normalized & [-10,10]   \\
    cost hidden layers & 3 & - & 3 \\
    cost activation & $\tanh$ & - & $\tanh$ \\
    cost hidden size & 100 & - & 100 \\
    cost step size $t^c$ & - & 0.05 & - \\
    Lipschits coefficient & 0 & - & 1.0 \\
    \hline
    entropy coefficient  & \multicolumn{3}{|c}{0.0} \\
    policy stepsize $t^\pi$  & \multicolumn{3}{|c}{0.5} \\
    cost update frequency (T) & \multicolumn{3}{|c}{2000} \\
    cost entropy coefficient & \multicolumn{3}{|c}{0.001} \\
    minibatch size & \multicolumn{3}{|c}{256} \\
    Adam stepsize & \multicolumn{3}{|c}{$3 \times 10^{-4}$} \\
    replay buffer size & \multicolumn{3}{|c}{$10^6$} \\
    target value function smoothing coefficient & \multicolumn{3}{|c}{0.01} \\
    mdpo update steps &
    \multicolumn{3}{|c}{10} \\
    number of policy hidden layers & \multicolumn{3}{|c}{2} \\
    discount factor & \multicolumn{3}{|c}{0.99} \\
    \#runs used for plot averages & \multicolumn{3}{|c}{5} \\
    confidence interval for plot runs & \multicolumn{3}{|c}{$\sim$ 95\%} \\
    \bottomrule
\end{tabular}}
\end{center}
\caption{Hyper-parameters of all off-policy methods.}
\label{tab:Hyperparams-off}
\end{table}